\documentclass[twoside,11pt]{article}
\usepackage{jmlr2e}

\usepackage{amsmath,amsbsy,amsfonts,amssymb, units}
\usepackage{graphicx}
\usepackage{tablefootnote}
\usepackage{color,cases}
\usepackage{tikz}
\usetikzlibrary{calc,shapes}
\usepackage{subcaption}
\usepackage{xcolor}
\definecolor{DarkGreen}{rgb}{0.1,0.5,0.1}
\definecolor{DarkRed}{rgb}{0.5,0.1,0.1}
\definecolor{DarkBlue}{rgb}{0.1,0.1,0.5}
\usepackage{hyperref}
\usepackage{algorithm}
\usepackage{algorithmic}
\usepackage{xspace}
\usepackage{wrapfig}

\numberwithin{equation}{section}
\numberwithin{theorem}{section}

\newcommand{\IGNORE}[1]{}

\newcommand\E{\mathbb{E}}
\newcommand{\qedhere}{~}

\newcommand\inner[1]{\langle #1 \rangle}

\newcommand\poly{\operatorname{poly}}

\newcommand{\Exp}{\mathop{\mathbb E}\displaylimits}
\newcommand{\tildO}[1]{\tilde{O}(#1)}

\newcommand{\unitcircle}{\mathcal{K}}

\newcommand{\conj}{^*}
\newcommand{\neighbor}{\mathcal{N}_{\tau}}

\newcommand{\hatG}{\widehat{G}}
\newcommand{\hats}{\hat{s}}
\newcommand{\hatr}{\hat{r}}

\newcommand{\thetax}{\theta}
\newcommand{\mper}{\,.}
\newcommand{\drq}{\Delta\rq}
\newcommand{\dr}{\Delta r}
\renewcommand{\rq}{r}

\newcommand{\hatp}{\hat{p}}
\newcommand{\hatA}{\hat{A}}
\newcommand{\hatB}{\hat{B}}
\newcommand{\hatC}{\hat{C}}
\newcommand{\hatD}{\hat{D}}
\newcommand{\hath}{\hat{h}}
\newcommand{\haty}{\hat{y}}

\newcommand{\hata}{\hat{a}}

\newcommand{\hatTheta}{\widehat{\Theta}}
\newcommand{\dy}{\Delta y}
\renewcommand{\dh}{\Delta h}

\newcommand{\CONDNAME}{{weakly-quasi-convex}\xspace}

\newcommand{\CONDSHORT}{{WQC}}
\newcommand{\ball}{\mathcal{B}}
\newcommand{\proj}[1]{\Pi_{#1}}

\newcommand\R{\mathbb{R}}
\newcommand\C{\mathbb{C}}

\newcommand{\Id}{I}
\newcommand{\ballC}{\mathcal{C}}

\newcommand{\spp}[1]{^{(#1)}}

\newcommand{\partialgrad}[2]{{\frac{\partial #1 }{\partial #2}}}
\newcommand{\tp}[2]{(#1^{\top})^{#2}}

\newcommand{\CC}{\textup{CC}}

\newcommand{\MCC}{\textup{MCC}}

\newcommand{\phat}{q}
\newcommand{\Var}{\textup{Var}}

\newcommand{\trace}{\textup{tr}}

\let\epsilon=\varepsilon

\numberwithin{equation}{section}

\newcommand\MYcurrentlabel{xxx}

\newcommand{\MYstore}[2]{%
	\global\expandafter \def \csname MYMEMORY #1 \endcsname{#2}%
}

\newcommand{\MYload}[1]{%
	\csname MYMEMORY #1 \endcsname%
}

\newcommand{\MYnewlabel}[1]{%
	\renewcommand\MYcurrentlabel{#1}%
	\MYoldlabel{#1}%
}

\newcommand{\MYdummylabel}[1]{}

\newcommand{\torestate}[1]{%
	\let\MYoldlabel\label%
	\let\label\MYnewlabel%
	#1%
	\MYstore{\MYcurrentlabel}{#1}%
	\let\label\MYoldlabel%
}

\newcommand{\restatetheorem}[1]{%
	\let\MYoldlabel\label
	\let\label\MYdummylabel
	\begin{theorem*}[Restatement of \prettyref{#1}]
		\MYload{#1}
	\end{theorem*}
	\let\label\MYoldlabel
}

\newcommand{\restatelemma}[1]{%
	\let\MYoldlabel\label
	\let\label\MYdummylabel
	\begin{lemma*}[Restatement of \prettyref{#1}]
		\MYload{#1}
	\end{lemma*}
	\let\label\MYoldlabel
}

\newcommand{\restateprop}[1]{%
	\let\MYoldlabel\label
	\let\label\MYdummylabel
	\begin{proposition*}[Restatement of \prettyref{#1}]
		\MYload{#1}
	\end{proposition*}
	\let\label\MYoldlabel
}

\newcommand{\restatefact}[1]{%
	\let\MYoldlabel\label
	\let\label\MYdummylabel
	\begin{fact*}[Restatement of \prettyref{#1}]
		\MYload{#1}
	\end{fact*}
	\let\label\MYoldlabel
}

\newcommand{\restate}[1]{%
	\let\MYoldlabel\label
	\let\label\MYdummylabel
	\MYload{#1}
	\let\label\MYoldlabel
}
\newcommand\indicator[1]{\ensuremath{\mathbf{1}_{#1}}}

\title{Gradient Descent Learns Linear Dynamical Systems}

\author{\name Moritz Hardt \email m@mrtz.org \\
	\addr Department of Electrical Engineering and Computer Science \\
	University of California, Berkeley
	\AND
	\name Tengyu Ma \email tengyuma@stanford.edu \\
	\addr Facebook AI Research
	\AND 
\name Benjamin Recht \email  brecht@berkeley.edu \\
	\addr Department of Electrical Engineering and Computer Science \\
University of California, Berkeley}

\usepackage{lastpage}
	
\jmlrheading{19}{2018}{1-\pageref{LastPage}}{9/16; Revised
		7/18}{8/18}{16-465}{Moritz Hardt, Tengyu Ma, and Benjamin Recht}
\ShortHeadings{Gradient Descent Learns Linear Dynamical Systems}{Hardt, Ma, and Recht}

\begin{document}
	\editor{Sujay Sanghavi}
\maketitle 
\begin{abstract}
We prove that stochastic gradient descent efficiently converges to the global optimizer of the maximum likelihood objective of an unknown linear time-invariant dynamical system from a sequence of noisy observations generated by the system. Even though the objective function is non-convex, we provide polynomial running time and sample complexity bounds under strong but natural assumptions. Linear systems identification has been studied for many decades, yet, to the best of our knowledge, these are the first polynomial guarantees for the problem we consider.

\vspace{0.15in}

\end{abstract}
\section{Introduction}
Many learning problems are by their nature sequence problems where the goal is
to fit a model that maps a sequence of input words $x_1,\dots,x_T$ to a corresponding sequence of observations $y_1,\dots,y_T.$ Text translation, speech recognition, time series prediction, video captioning and question answering systems, to name a few, are all sequence to sequence learning problems. For a sequence model to be both expressive and parsimonious in its parameterization, it is crucial to equip the model with memory thus allowing its prediction at time~$t$ to depend on previously seen inputs. 

Recurrent neural networks form an expressive class of non-linear sequence models. Through their many variants, such as long-short-term-memory~\cite{lstm}, recurrent neural networks have seen remarkable empirical success in a broad range of domains. At the core, neural networks are typically trained using some form of (stochastic) gradient descent. Even though the training objective is non-convex, it is widely observed in practice that gradient descent quickly approaches a good set of model parameters. Understanding the effectiveness of gradient descent for non-convex objectives on theoretical grounds is a major open problem in this area. 

If we remove all non-linear state transitions from a recurrent neural network, we are left with the state transition representation of a linear dynamical system. Notwithstanding, the natural training objective for linear systems remains non-convex due to the composition of multiple linear operators in the system. If there is any hope of eventually understanding recurrent neural networks, it will be inevitable to develop a solid understanding of this special case first.

To be sure, linear dynamical systems are important in their own right and have been studied for many decades independently of machine learning within the control theory community. Control theory provides a rich set techniques for identifying and manipulating linear systems. The learning problem in this context corresponds to ``linear dynamical system identification''. Maximum likelihood estimation with gradient descent is a popular heuristic for dynamical system identification~\cite{LjungBook}.  In the context of machine learning, linear systems play an important role in numerous tasks.  For example, their estimation arises as subroutines of reinforcement learning in robotics~\cite{levine2013guided}, location and mapping estimation in robotic systems~\cite{durrant2006simultaneous}, and estimation of pose from video~\cite{Rahimi05}.

In this work, we show that gradient descent efficiently minimizes the maximum likelihood objective of an unknown linear system given noisy observations generated by the system. More formally, we receive noisy observations generated by the following \emph{time-invariant linear system}:
\begin{align}
h_{t+1} & = Ah_{t} + Bx_t \label{eqn:model} \\
y_t & = Ch_t + Dx_t + \xi_t\nonumber
\end{align}
Here, $A,B,C,D$ are linear transformations with compatible dimensions and we denote by $\Theta=(A,B,C,D)$ the parameters of the system. The vector $h_t$ represents the hidden state of the model at time~$t.$ Its dimension~$n$ is called the \emph{order} of the system. The stochastic noise variables $\{\xi_t\}$ perturb the output of the system which is why the model is called an \emph{output error model} in control theory. We assume the variables are drawn i.i.d.~from a distribution with mean~$0$ and variance~$\sigma^2.$

Throughout the paper we focus on \emph{controllable} and \emph{externally stable} systems. A linear system is externally stable (or \emph{equivalently bounded-input bounded-output stable}) if and only if the spectral radius of~$A$, denoted $\rho(A),$ is strictly bounded by~$1.$ Controllability is a mild non-degeneracy assumption that we formally define later.
Without loss of generality, we further assume that the transformations $B$,~$C$ and~$D$ have bounded Frobenius norm. This can be achieved by a rescaling of the output variables.
We assume we have $N$ pairs of sequences $(x,y)$ as training examples,
\[
S = \left\{(x\spp{1},y\spp{1}),\dots,(x\spp{N},y\spp{N})\right\}\,.
\]
Each input sequence~$x\in \R^T$ of length $T$ is drawn from a distribution and $y$ is the corresponding output of the system above generated from an unknown initial state $h.$ We allow the unknown initial state to vary from one input sequence to the next. This only makes the learning problem more challenging.

Our goal is to fit a linear system to the observations. We parameterize our model in exactly the same way as~\eqref{eqn:model}. That is, for linear mappings $(\hatA,\hatB,\hatC,\hatD)$, the trained model is defined as:
\begin{align}
\hath_{t+1} = \hatA\hath_{t} + \hatB x_t\,,\qquad
\haty_t = \hatC\hath_t + \hatD x_t \label{eqn:haty_t}
\end{align}

The \emph{(population) risk} of the model is obtained by feeding the learned system with the correct initial states and comparing its predictions with the ground truth in expectation over inputs and errors. Denoting by $\hat y_t$ the $t$-th prediction of the trained model starting from the correct initial state that generated $y_t$, and using $\hatTheta$ as a short hand for $(\hatA,\hatB,\hatC,\hatD)$, we formally define population risk as:
\begin{equation}\label{eq:pop}
f(\hatTheta) = \Exp_{\{x_t\},\{\xi_t\}}\left[\frac 1T
\sum_{t=1}^{T}\left\|\haty_t -
y_t\right\|^2\right]
\end{equation}

Note that even though the prediction $\hat y_t$ is generated from the correct initial state, the learning algorithm does not have access to the correct initial state for its training sequences.

While the squared loss objective turns out to be non-convex, it has many appealing properties. Assuming the inputs $x_t$ and errors~$\xi_t$ are drawn independently from a Gaussian distribution, the corresponding population objective corresponds to maximum likelihood estimation. In this work, we make the weaker assumption that the inputs and errors are drawn independently from possibly different distributions. The independence assumption is certainly idealized for some learning applications. However, in control applications the inputs can often be chosen by the controller rather than by nature. Moreover, the outputs of the system at various time steps are correlated through the unknown hidden state and therefore not independent even if the inputs are.

\subsection{Our results}

We show that we can efficiently minimize the population risk using projected stochastic gradient descent. 
The bulk of our work applies to single-input single-output (SISO) systems meaning that inputs and outputs are scalars $x_t,y_t\in\mathbb{R}$. However, the hidden state can have arbitrary dimension~$n$. Every controllable SISO admits a convenient canonical form called \emph{controllable canonical form} that we formally introduce later in Section~\ref{sec:prelim}. In this canonical form, the transition matrix $A$ is governed by $n$ parameters $a_1,\dots,a_n$ which coincide with the coefficients of the characteristic polynomial of~$A.$ The minimal assumption under which we might hope to learn the system is that the spectral radius of~$A$ is smaller than~$1.$ However, the set of such matrices is non-convex and does not have enough structure for our analysis.\footnote{In both the controllable canonical form and the standard parameterization of the matrix $A$, the set of matrices with spectral radius less than 1 is not convex.} We will therefore make additional assumptions. The assumptions we need differ between the case where we are trying to learn $A$ with $n$ parameter system, and the case where we allow ourselves to over-specify the trained model with $n'> n$ parameters. The former is sometimes called proper learning, while the latter is called improper learning. In the improper case, we are essentially able to learn any system with spectral radius less than~$1$ under a mild separation condition on the roots of the characteristic polynomial. Our assumption in the proper case is stronger and we introduce it next.

\subsection{Proper learning}

Suppose the state transition matrix $A$ has characteristic polynomial $\det(zI - A) = z^n + a_1z^{n-1}+\dots + a_n$, which turns out to by and large decide the difficulty of the learning according to our analysis. (In fact, we will parameterize $A$ in a way so that the coefficients of the characteristic polynomials are the parameters of learning problem. See Section~\ref{sec:prelim} for the detailed setup.)
%Suppose that the state transition matrix~$A$ is given by parameters $a_1,\dots,a_n$ and 
Consider the corresponding  polynomial $q(z)= 1 + a_1z+a_2z^2+\cdots +a_nz^n$ over the complex numbers~$\mathbb{C}.$

\begin{wrapfigure}{R}{0.4\textwidth}
\includegraphics[width=0.38\textwidth]{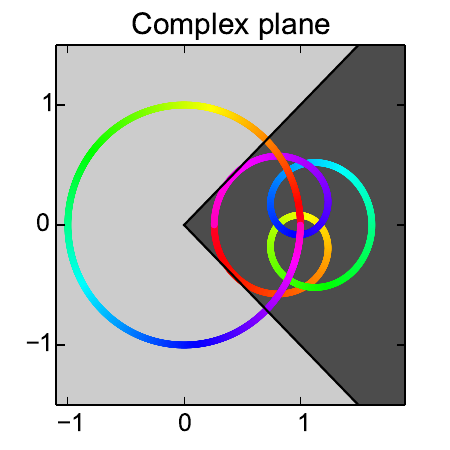}
\caption{An example of polynomial $q$ that satisfies our assumption. The unit circle is the collection of the inputs of $q$ and the other curve shows the corresponding outputs (with the corresponding colors.) We see the image of the polynomial stays in the wedge which contains all the complex number $z$ satisfying $\Re(q(z))>|\Im(q(z))|.$ }\label{fig:1}
\end{wrapfigure}

We will require the state transition matrix satisfy that the image of the unit circle on the complex plane under the polynomial~$q$ is contained in the cone of complex numbers whose real part is larger than their absolute imaginary part. Formally, for all $z\in\mathbb{C}$ such that $|z|=1,$ we require that $\Re(q(z))>|\Im(q(z))|.$ Here, $\Re(z)$ and $\Im(z)$ denote the real and imaginary part of $z,$ respectively. We illustrate this condition in Figure~\ref{fig:1} on the right for a degree~$4$ system.

Our assumption has three important implications. First, it implies (via Rouch\'{e}'s theorem) that the spectral radius of~$A$ is smaller than~$1$ and therefore ensures the stability of the system. Second, the vectors satisfying our assumption form a convex set in~$\mathbb{R}^n.$ Finally, it ensures that the objective function is \emph{weakly quasi-convex}, a condition we introduce later when we show that it enables stochastic gradient descent to make sufficient progress.

We note in passing that our assumption can be satisfied via the $\ell_1$-norm constraint $\|a\|_1\le \sqrt{2}/2.$ Moreover, if we pick random Gaussian coefficients with expected norm bounded by~$o(1/\sqrt{\log n}),$ then the resulting vector will satisfy our assumption with probability $1-o(1)$. 
Roughly speaking, the assumption requires the roots of the characteristic polynomial $p(z) = z^n+a_1z^{n-1}+\dots + a_n$ are relatively dispersed inside the unit circle. 
(For comparison, on the other end of the spectrum, the polynomial $p(z) = (z-0.99)^n$ have all its roots colliding at the same point and doesn't satisfy the assumption.)
\begin{theorem}[Informal]\label{thm:proper-intro}
	Under our assumption, projected stochastic gradient descent, when given~$N$ sample sequence of length~$T$, returns parameters~$\hatTheta$ with population risk 
			\begin{equation}
	\Exp f(\hatTheta) \le f(\Theta) + O\left(\sqrt{\frac{n^5+\sigma^2n^3}{TN}}\right)\mper\nonumber
	\end{equation}
\end{theorem}

Here the expectation on LHS is with respect to the randomness of the algorithm. Recall that $f(\Theta)$ is the population risk of the optimal system, and $\sigma^2$ refers to the variance of the noise variables. We also assume that the inputs $x_t$ are drawn from a pairwise independent distribution with mean~$0$ and variance~$1.$ Note, however, that this does not imply independence of the outputs as these are correlated by a common hidden state.
The stated version of our result glosses over the fact that we need our assumption to hold with a small amount of slack; a precise version follows in Section~\ref{sec:acq}. Our theorem establishes a polynomial convergence rate for stochastic gradient descent. Since each iteration of the algorithm only requires a sequence of matrix operations and an efficient projection step, the running time is polynomial, as well. Likewise, the sample requirements are polynomial since each iteration requires only a single fresh example. An important feature of this result is that the error decreases with both the length $T$ and the number of samples $N$. The dependency on the dimension $n$, on the other hand, is likely to be quite loose, and tighter bounds are left for future work. 

The algorithm requires a (polynomial-time) projection step to a convex set at every iteration (formally defined in Section~\ref{sec:acq} and Algorithm~\ref{alg:backprop}). Computationally, it can be a bottleneck, although it is unlikely to be required in practice and may be an artifact of our analysis.

\subsection{The power of over-parameterization}

Endowing the model with additional parameters compared to the ground truth turns out to be surprisingly powerful. We show that we can essentially remove the assumption we previously made in  proper learning. The idea is simple. If $p$ is the characteristic polynomial of $A$ of degree $n$. We can find a system of order $n'>n$ such that the characteristic polynomial of its transition matrix becomes $p \cdot p'$ for some polynomial $p'$ of order $n'-n.$ This means that to apply our result we only need the polynomial $p\cdot p'$ to satisfy our assumption. At this point, we can choose $p'$ to be an approximation of the inverse $p^{-1}$. For sufficiently good approximation, the resulting polynomial $p\cdot p'$ is close to $1$ and therefore satisfies our assumption. Such an approximation exists generically for $n'=O(n)$ under mild non-degeneracy assumptions on the roots of~$p$. In particular, any small random perturbation of the roots would suffice.

\begin{theorem}[Informal]\label{thm:improper-intro}
	Under a mild non-degeneracy assumption, stochastic gradient descent returns parameters~$\hatTheta$ corresponding to a system of order $n'=O(n)$ with population risk 
		\begin{equation}
	f(\hatTheta) \le f(\Theta) + O\left(\sqrt{\frac{n^5+\sigma^2n^3}{TN}}\right)\,,\nonumber
	\end{equation}
	when given $N$ sample sequences of length~$T$. 
\end{theorem}

We remark that the idea we sketched also shows that, in the extreme, improper learning of linear dynamical systems becomes easy in the sense that the problem can be solved using linear regression against the outputs of the system. 		However, our general result interpolates between the proper case and the regime where linear regression works. We discuss in more details in Section~\ref{subsec:linear_regression}. 
\subsection{Multi-input multi-output systems}

Both results we saw immediately extend to single-input multi-output (SIMO) systems as the dimensionality of~$C$ and $D$ are irrelevant for us. The case of multi-input multi-output (MIMO) systems is more delicate. Specifically, our results carry over to a broad family of multi-input multi-output systems. However, in general MIMO systems no longer enjoy canonical forms like SISO systems. In Section~\ref{sec:mimo}, we introduce a natural generalization of controllable canonical form for MIMO systems and extend our results to this case.

\subsection{Related work}

System identification is a core problem in dynamical systems and has been studied in depth for many years.  The most popular reference on this topic is the text by Ljung~\cite{LjungBook}.  Nonetheless, the list of non-asymptotic results on identifying linear systems from noisy data is surprisingly short.
Several authors have recently tried to estimate the sample complexity of dynamical system identification using machine learning tools~\cite{Vidyasagar08, Campi02, Weyer99}.  All of these result are rather pessimistic with sample complexity bounds that are exponential in the degree of the linear system and other relevant quantities.  Contrastingly,  we prove that gradient descent has an associated polynomial sample complexity in all of these parameters.  Moreover, all of these papers only focus on how well empirical risk approximates the true population risk and do not provide guarantees about any algorithmic schemes for minimizing the empirical risk.

The only result to our knowledge which provides polynomial sample complexity for identifying linear dynamical systems is in Shah \emph{et al}~\cite{Shah12}.  Here, the authors show that if certain \emph{frequency domain} information about the linear dynamical system is observed, then the linear system can be identified by solving a second-order cone programming problem. This result is about improper learning only, and the size of the resulting system may be quite large, scaling as the $(1-\rho(A))^{-2}$.  As we describe in this work, very simple algorithms work in the improper setting when the system degree is allowed to be polynomial in $(1-\rho(A))^{-1}$.  Moreover, it is not immediately clear how to translate the frequency-domain results to the time-domain identification problem discussed above.

Our main assumption about the image of the polynomial $q(z)$ is an appeal to the theory of passive systems.  A system is passive if the dot product between the input sequence $u_t$ and output sequence $y_t$ are strictly positive.  Physically, this notion corresponds to systems that cannot create energy.  For example, a circuit made solely of resistors, capacitors, and inductors would be a passive electrical system.  If one added an amplifier to the internals of the system, then it would no longer be passive.  The set of passive systems is a subset of the set of stable systems, and the subclass is somewhat easier to work with mathematically.  Indeed, Megretski used tools from passive systems to provide a relaxation technique for a family of identification problems in dynamical systems~\cite{Megretski08}.  His approach is to lower bound a nonlinear least-squares cost with a convex functional.  However, he does not prove that his technique can identify any of the systems, even asymptotically. ~\citet{soderstrom1982some, stoica1982uniqueness,stoica1984uniqueness} and later \citet{Bazanella08,eckhard2011global} prove the quasi-convexity of a cost function under a passivity condition in the context of system identification, but no sample complexity or global convergence proofs are provided.

\subsection{Proof overview}\label{sec:proof_overview}

The first important step in our proof is to develop population risk in Fourier domain where it is closely related to what we call \emph{idealized risk}. Idealized risk essentially captures the $\ell_2$-difference between the \emph{transfer function} of the learned system and the ground truth. The transfer function is a fundamental object in control theory. Any linear system is completely characterized by its transfer function $G(z) = C(zI - A)^{-1}B.$ In the case of a SISO, the transfer function is a rational function of degree $n$ over the complex numbers and can be written as $G(z)=s(z)/p(z).$ In the canonical form introduced in Section~\ref{sec:prelim}, the coefficients of $p(z)$ are precisely the parameters that specify $A.$ 
Moreover, $z^n p(1/z)=1+a_1z+a_2z^2+\cdots +a_nz^n$ is the polynomial we encountered in the introduction. Under the assumption illustrated earlier, we show in Section~\ref{sec:fourier} that the idealized risk is \emph{weakly quasi-convex} (Lemma~\ref{lem:main}). Quasi-convexity implies that gradients cannot vanish except at the optimum of the objective function; we review this (mostly known) material in Section~\ref{sec:quasi}. 
In particular, this lemma implies that in principle we can hope to show that gradient descent converges to a global optimum. However, there are several important issues that we need to address. First, the result only applies to idealized risk, not our actual population risk objective. Therefore it is not clear how to obtain unbiased gradients of the idealized risk objective. Second, there is a subtlety in even defining a suitable empirical risk objective. The reason is that risk is defined with respect to the correct initial state of the system which we do not have access to during training.
We overcome both of these problems. In particular, we design an almost unbiased estimator of the gradient of the idealized risk in Lemma~\ref{lem:unbiased_estimator} and give variance bounds of the gradient estimator (Lemma~\ref{lem:variance}).

Our results on improper learning in Section~\ref{sec:improper} rely on a surprisingly simple but powerful insight. We can extend the degree of the transfer function $G(z)$ by extending both numerator and denominator with a polynomial $u(z)$ such that $G(z)=s(z)u(z)/p(z)u(z).$ While this results in an equivalent system in terms of input-output behavior, it can dramatically change the geometry of the optimization landscape. In particular, we can see that only $p(z)u(z)$ has to satisfy the assumption of our proper learning algorithm. This allows us, for example, to put $u(z)\approx p(z)^{-1}$ so that $p(z)u(z)\approx 1,$ hence trivially satisfying our assumption. A suitable inverse approximation exists under light assumptions and requires degree no more than $d=O(n).$ Algorithmically, there is almost no change. We simply run stochastic gradient descent with $n+d$ model parameters rather than $n$ model parameters.

\subsection{Preliminaries}
\label{sec:prelim}

For complex matrix (or vector, number) $C$, we use $\Re(C)$ to denote the real part and  $\Im(C)$ the imaginary part, and $\bar{C}$ the conjugate and $C^{*} = \bar{C}^{\top}$ its conjugate transpose . We use $|\cdot|$ to denote the absolute value of a complex number $c$. For complex vector $u$ and $v$, we use $\inner{u,v} = u^*v$ to denote the inner product and 
$\|u\| = \sqrt{u^*u}$ is the norm of $u$. 
For complex matrix $A$ and $B$ with same dimension, $\inner{A,B} = \trace(A^*B)$ defines an inner product, and $\|A\|_F = \sqrt{\trace(A^*A)}$ is the Frobenius norm. 
For a square matrix $A$, we use $\rho(A)$ to
denote the spectral radius of $A$, that is, the largest absolute value of the
elements in the spectrum of $A$. We use $\Id_n$ to denote the identity matrix
with dimension $n\times n$, and we drop the subscript when it's clear from the
context.We let $e_i$ denote the $i$-th standard basis vector.

A SISO of order~$n$ is in \emph{controllable canonical form} if~$A$ and $B$
have the following form
\begin{equation}
A\; = \;
	\left[ \begin{array}{ccccc} 0 & 1 & 0 & \cdots & 0 \\ 0 & 0 & 1 & \cdots & 0 \\
	\vdots & \vdots & \vdots & \ddots & \vdots \\  0 & 0 & 0 & \cdots & 1 \\
	-a_n & -a_{n-1} & -a_{n-2} & \cdots & -a_1 \end{array} \right]
\qquad 
B = \left[ \begin{array}{c} 0\\ 0 \\ \vdots \\ 0 \\ 1 \end{array} \right]
\label{eqn:controllable_form}
	\end{equation}
 
We will parameterize $\hatA,\hatB,\hatC,\hatD$ accordingly. We will write $A = \CC(a)$ for brevity, where $a$ is used to denote the unknown last row $[-a_n, \dots, -a_1]$ of matrix $A$. We will use $\hata$ to denote the corresponding training variables for $a$.  Since here $B$ is known, so $\hatB$ is no longer a trainable parameter, and is forced to be equal to $B$. Moreover, $C$ is a row vector and we use $[c_1, \cdots, c_n]$ to denote its coordinates (and similarly for $\hatC$).

A SISO is \emph{controllable} if and only if the matrix 
$[B\mid AB\mid A^2B\mid\cdots\mid A^{n-1}B]$ has rank~$n.$ This statement corresponds to the condition that all hidden states should be reachable from some initial condition and input trajectory. Any controllable system admits a controllable canonical form~\cite{Heij07}. 
For vector $a = [a_n,\dots,a_1]$, let $p_a(z)$ denote the polynomial
\begin{equation}
p_a(z) = z^n + a_1z^{n-1}+\dots+a_n\,.\label{eqn:def_p}
\end{equation}
When $a$ defines the matrix~$A$ that appears in controllable canonical form,
then $p_a$ is precisely the characteristic polynomial of~$A.$ That is,
$p_a(z)=\mathrm{det}(z I - A).$

\section{Gradient descent and quasi-convexity}
\label{sec:quasi}

It is known that under certain mild conditions (stochastic) gradient descent converges even on non-convex functions to local minimum~\cite{GeEscaping,LeeGradient}. Though usually for concrete problems the challenge is to prove that there is no spurious local minimum other than the target solution. Here
we introduce a condition similar to the quasi-convexity notion in~\cite{HazanBeyond}, which ensures that any point with vanishing gradient is the optimal solution . 
Roughly speaking, the condition says that at any point~$\theta$ the negative of the gradient $-\nabla f(\theta)$ should be positively correlated with direction~$\thetax^*-\theta$ pointing towards the optimum. Our condition is slightly weaker than that in~\cite{HazanBeyond} since we only require quasi-convexity and smoothness with respect to the optimum, and this (simple) extension will be necessary for our analysis. 	

\begin{definition}[Weak quasi-convexity]\label{def:cond}
We say an objective function $f$ is \emph{$\tau$-\CONDNAME} ($\tau$-\CONDSHORT) 
over a domain~$\ball$ with respect to global minimum $\thetax^*$ if there is a
positive constant $\tau>0$ such that for all $\theta\in \ball$, 
\begin{equation}
	\nabla f(\thetax)^{\top} (\thetax-\thetax^*) \ge \tau (f(\thetax)-f(\thetax^*))\,.\label{eqn:condition}
\end{equation}
We further say $f$ is $\Gamma$-weakly-smooth if for for any point $\thetax$, $\|\nabla f(\thetax)\|^2 \le \Gamma (f(\thetax) - f(\thetax^*))$.
\end{definition}
Note that indeed any $\Gamma$-smooth function in the usual sense (that is, $\|\nabla^2 f\|\le \Gamma$) is $O(\Gamma)$-weakly-smooth. For a random vector $X\in \R^n$, we define it's variance to be $\Var[X] = \Exp[\|X -\E X\|^2]$. 

\newcommand{\mfr}{\mathfrak{r}}
\begin{definition}
We call $\mfr(\theta)$ an unbiased estimator of  $\nabla f(\thetax)$ with variance $V$ if it satisfies 
$\Exp[\mfr(\theta)] = \nabla f(\thetax)$ and $\Var[\mfr(\theta)] \le V$.
\end{definition}

\emph{Projected stochastic gradient descent} over some closed convex set~$\ball$  with \emph{learning rate}~$\eta>0$ refers to the following algorithm in which
$\proj{\ball}$ denotes the Euclidean projection onto~$\ball$: 
\begin{align}
	&\textrm{ {\bf for} $k = 0$ to $K-1$ : }\nonumber\\
	& \quad \quad \quad w_{k+1}  = \thetax_k -\eta \mfr(\thetax_k)\nonumber \\
	& \quad \quad \quad \thetax_{k+1}   = \proj{\ball}(w_{k+1})\nonumber \\
	& \textrm{ {\bf return}  $\thetax_{j}$ with $j$ uniformly picked from ${1,\dots,K}$}\label{eqn:grad-descent}
\end{align}

	The following Proposition is well known for convex objective functions (corresponding to $1$-\CONDNAME functions). We extend it (straightforwardly) to the case when $\tau$-\CONDSHORT~holds with any positive constant $\tau$. 
	
	\begin{proposition}\label{lem:framework}
	Suppose the objective function $f$ is $\tau$-\CONDNAME and $\Gamma$-weakly-smooth, and $\mfr(\cdot)$ is an unbiased estimator for $\nabla f(\thetax)$ with variance $V$. Moreover, suppose the global minimum $\thetax^*$ belongs to $\ball$, and the initial point $\thetax_0$ satisfies $\|\thetax_0-\thetax^*\|\le R$. Then projected gradient descent~\eqref{eqn:grad-descent} with a proper learning rate 
	returns $\theta_K$ in $K$ iterations with expected error $$\Exp f(\theta_K) -f(\theta^*)\le O\left(\max\left\{\frac{\Gamma R^2}{\tau^2K}, \frac{R\sqrt{V}}{\tau \sqrt{K}}\right\}\right)\mper$$ 
	\end{proposition}
	\begin{remark}
		It's straightforward to see (from the proof) that the algorithm tolerates inverse exponential bias, namely bias on the order of $\exp(-\Omega(n))$, in the gradient estimator. Technically, suppose $\Exp[\mfr(\theta)] = \nabla f(\theta) \pm \zeta$ then $f(\theta_K)\le O\left(\max\left\{\frac{\Gamma R^2}{\tau^2K}, \frac{R\sqrt{V}}{\tau \sqrt{K}}\right\}\right) + \poly(K)\cdot \zeta$.  Throughout the paper, we assume  that the error that we are shooting for is inverse polynomial, namely $1/n^C$ for some absolute constant $C$, and therefore the effect of inverse exponential bias is negligible.
	\end{remark}

We defer the proof of Proposition~\ref{lem:framework} to Appendix~\ref{sec:optimization} which is a simple variation of the standard convergence analysis of stochastic gradient descent (see, for example, ~\cite{Bottou:1999:OLS:304710.304720}).
	Finally, we note that  the sum of two quasi-convex functions may no longer be quasi-convex. However, if a sequence functions is $\tau$-\CONDSHORT~with respect to a common point $\thetax^*$, then their sum is also $\tau$-\CONDSHORT. This follows from the linearity of gradient operation.
	
\begin{proposition}\label{prop:QWC-sum}Suppose functions $f_1,\dots, f_n$ are individually $\tau$-\CONDNAME~in $\ball$ with respect to a common global minimum $\thetax^*$ , then for non-negative $w_1,\dots, w_n$ the linear combination $f = \sum_{i=1}^n w_if_i$ is also $\tau$-\CONDNAME~with respect to $\thetax^*$ in $\ball$. 
\end{proposition}

\section{Population risk in frequency domain}
\label{sec:fourier}

We next establish conditions under which risk is \CONDNAME. Our strategy is to first approximate the risk functional~$f(\hatTheta)$ by what we call \emph{idealized risk}. This approximation of the objective function is fairly straightforward; we justify it toward the end of the section. We can show that \begin{equation}\textstyle
f(\hatTheta) \approx \|D-\hatD\|^2 + \sum_{k=0}^{\infty} \big(\hatC\hatA^k B - CA^kB\big)^2 \mper
\label{eqn:approximation}
\end{equation}

The leading term $\|D-\hatD\|^2$ is convex in $\hatD$ which appears
nowhere else in the objective. It therefore doesn't affect the convergence of
the algorithm (up to lower order terms) by virtue of Proposition~\ref{prop:QWC-sum}, and we restrict our attention to the remaining terms.

\begin{definition}[Idealized risk]
We define the \emph{idealized risk} as
\begin{equation}
g(\hatA,\hatC) = \sum_{k=0}^{\infty} \left(\hatC \hatA^{k}B-
CA^{k}B\right)^2\,. \label{eqn:ideal-objective}
\end{equation}
\end{definition}

We now use basic concepts from control theory (see~\cite{Heij07,hespanha2009} for more background) to express the idealized risk~\eqref{eqn:ideal-objective} in Fourier domain.
The \emph{transfer function} of the linear system is
\begin{equation}\label{eq:transfer-fn-def}
	G(z) = C(zI-A)^{-1}B\mper
\end{equation}

Note that $G(z)$ is a rational function over the complex numbers of degree $n$ and hence we can find polynomials $s(z)$ and $p(z)$ such that
$G(z) = \frac{s(z)}{p(z)} \,,$
with the convention that the leading coefficient of $p(z)$ is $1$.  
In controllable canonical form~\eqref{eqn:controllable_form}, the coefficients of $p$ will correspond to the last row of the $A,$ while the coefficients of $s$ correspond to the entries of $C$. 
Also note that
\[
G(z) = \sum_{t=1}^\infty z^{-t} CA^{t-1} B = \sum_{t=1}^\infty z^{-t} r_{t-1}
\]
The sequence $r=(r_0,r_1,\ldots, r_t,\ldots) = (CB, CAB, \ldots, CA^{t}B,\ldots)$ is called the \emph{impulse response} of the linear system. The behavior of a linear system is uniquely determined by the impulse response and therefore by $G(z).$ 
Analogously, we denote the transfer function of the learned system by
$
\hatG(z) = \hatC(zI-\hatA)^{-1}B = \hats(z)/\hatp(z)\,.
$
The idealized risk~\eqref{eqn:ideal-objective} is only a function of the impulse response $\hatr$ of the learned system, and therefore it is only a function of $\hatG(z)$. 

Recall that $C = [c_1,\dots, c_n]$ is defined in Section~\ref{sec:prelim}. For future reference, we note that by some elementary calculation (see Lemma~\ref{lem:uinverseB}), we have
	\begin{equation}
	G(z) = C(zI-A)^{-1}B = \frac{c_1+\dots + c_nz^{n-1}}{z^n+a_1z^{n-1}+\dots + a_n}\,,
	\end{equation}
	which implies that $s(z)=c_1+\dots + c_nz^{n-1}$ and $p(z)=z^n+a_1z^{n-1}+\dots + a_n$.

With these definitions in mind, we are ready to express idealized risk in Fourier domain.

\begin{proposition}\label{prop:risk_freq_rep}
	Suppose $p_{\hata}(z)$ has all its roots inside unit circle, then the idealized risk $g(\hata, \hatC)$ can be written in the Fourier domain as
	\[
	g(\hatA,\hatC)= \int_{0}^{2\pi}\left|\hatG(e^{i\theta})- G( e^{i\theta})\right|^2\mathrm{d}\theta\mper
	\]
\end{proposition}
\begin{proof}
	Note that $G(e^{i\theta})$ is the Fourier transform of the sequence $\{r_k\}$ and so is $\hatG(e^{i\theta})$ the Fourier transform\footnote{The Fourier transform exists since $\|r_k\|^2 = \|\hatC\hatA^k\hatB\|^2 \le \|\hatC\| \|\hatA^k\|\|\hatB\|\le c\rho(\hatA)^k$ where $c$ doesn't depend on $k$ and $\rho(\hatA) < 1$.}  of $\hatr_k$. Therefore by Parseval' Theorem, we have that 
	\begin{align}
	g(\hatA,\hatC) &= \sum_{k=0}^{\infty} \|\hatr_k-r_k \|^2 = \int_{0}^{2\pi} |\hatG(e^{i\theta})- G( e^{i\theta})|^2d\theta\,.
	\end{align}
\end{proof}

\subsection{Quasi-convexity of the idealized risk}

Now that we have a convenient expression for risk in Fourier domain, we can prove that the idealized risk $g(\hatA,\hatC)$ is weakly-quasi-convex when $\hata$ is not so far from the true $a$ in the sense that $p_a(z)$ and $\hatp_a(z)$ have an angle less than $\pi/2$ for every $z$ on the unit circle. We will use the convention that $a$ and $\hata$ refer to the parameters that specify $A$ and $\hatA,$ respectively.

\begin{lemma}
\label{lem:main}
\label{lem:main_single_term}
	For $\tau>0$ and every $\hat C$, the idealized risk $g(\hatA,\hatC)$ is $\tau$-weakly-quasi-convex over the domain 
	\begin{equation}\label{eqn:main-cond}
	\neighbor(a)= \left\{
	\hata\in\mathbb{R}^n\colon 
	\Re\left(\frac{p_{a}(z)}{p_{\hata}(z)}\right) \ge \tau/2, \forall~z\in \C, \textup{~s.t.~} |z| = 1
	\right\}
	\,.
	\end{equation}
\end{lemma}

\begin{proof}
We first analyze a single term $h= |\hatG(z)-G(z)|^2$. Recall that $\hatG(z) = \hats(z)/\hatp(z)$ where $\hatp(z) = p_{\hata}(z) = z^n+\hata_1z^{n-1}+\dots + \hata_n$. Note that $z$ is fixed and $h$ is a function of $\hata$ and $\hatC$. 
	Then it is straightforward to see that 
	\begin{equation}
		\partialgrad{h}{\hats(z)} =  2\Re\left\{\frac{1}{\hatp(z)}\left(\frac{\hats(z)}{\hatp(z)}-\frac{s(z)}{p(z)}\right)\conj\right\}\label{eqn:grads}\mper
	\end{equation}
	
	and 
	\begin{equation}
		\partialgrad{h}{\hatp(z)} =  -2\Re\left\{\frac{\hats(z)}{\hatp(z)^2}\left(\frac{\hats(z)}{\hatp(z)}-\frac{s(z)}{p(z)}\right)\conj\right\}\label{eqn:gradp}\mper
	\end{equation}
	Since $\hats(z)$ and $\hatp(z)$ are linear in $\hatC$ and $\hata$ respectively, by chain rule we have that  
	\begin{align*}
\inner{\partialgrad{h}{\hata}, \hata - a}+ \inner{\partialgrad{h}{\hatC}, \hatC - C} & = \partialgrad{h}{\hatp(z)}\inner{\partialgrad{\hatp(z)}{\hata}, \hata - a} + \partialgrad{h}{\hats(z)}\inner{\partialgrad{\hats(z)}{\hatC}, \hatC - C} \nonumber\\
& = \partialgrad{h}{\hatp(z)}(\hatp(z)-p(z))+ \partialgrad{h}{\hats(z)}(\hats(z)-s(z))\mper
	\end{align*}
	Plugging the formulas~\eqref{eqn:grads} and~\eqref{eqn:gradp} for $\partialgrad{h}{\hats(z)}$ and $\partialgrad{h}{\hatp(z)}$ into the equation above, we obtain that 
	\begin{align}
		\inner{\partialgrad{h}{\hata}, \hata - a}+ \inner{\partialgrad{h}{\hatC}, \hatC - C} & = 2 \Re \left\{ \frac{-\hats(z)(\hatp(z)-p(z))+ \hatp(z)(\hats(z)-s(z))}{\hatp(z)^2}    \left(\frac{\hats(z)}{\hatp(z)}  - \frac{s(z)}{p(z)}  \right)^*\right\}\nonumber\\
		&= 2 \Re \left\{ \frac{\hats(z)p(z)-s(z)\hatp(z) }{\hatp(z)^2}    \left(\frac{\hats(z)}{\hatp(z)}  - \frac{s(z)}{p(z)}  \right)^*\right\}\nonumber\\
		&= 2 \Re \left\{ \frac{p(z)}{\hatp(z)}\left|\frac{\hats(z)}{\hatp(z)}  - \frac{s(z)}{p(z)}  \right|^2\right\} \nonumber\\
		& = 2 \Re \left\{ \frac{p(z)}{\hatp(z)}\right\} \left|\hatG(z)-G(z)\right|^2\mper\nonumber
	\end{align}

Hence $h = |\hatG(z)-G(z)|^2 $ is $\tau$-weakly-quasi-convex with $\tau= 2 \min_{|z|=1} \Re\left\{\frac{p(z)}{\hatp(z)}\right\}$. This implies our claim, since by Proposition~\ref{prop:risk_freq_rep}, the idealized risk $g$ is convex combination of functions of the form $|\hatG(z)-G(z)|^2$ for $|z| = 1$. Moreover, Proposition~\ref{prop:QWC-sum} shows that convex combination preserves weak quasi-convexity.
\end{proof}

For future reference, we also prove that the idealized risk is $O(n^2/\tau_1^4)$-weakly smooth. 

\begin{lemma}\label{lem:smooth}
	The idealized risk $g(\hatA,\hatC)$ is $\Gamma$-weakly smooth  with $\Gamma = O(n^2/\tau_1^4)$.  
\end{lemma}
\begin{proof}
	By equation~\eqref{eqn:gradp} and the chain rule we get that 
	\begin{align}
	\partialgrad{g}{\hatC} & = \int_{\mathbb{T}} \partialgrad{|\hatG(z)-G(z)|^2}{p(z)}\cdot\partialgrad{p(z)}{\hatC}dz = \int_{\mathbb{T}}2\Re\left\{\frac{1}{\hatp(z)}\left(\frac{\hats(z)}{\hatp(z)}-\frac{s(z)}{p(z)}\right)\conj\right\}\cdot [1,\dots, z^{n-1}]dz\mper\nonumber
	\end{align}
	therefore we can bound the norm of the gradient by 
	\begin{align}
	\left\|\partialgrad{g}{\hatC}\right\|^2  & \le \left(\int_{\mathbb{T}} \left|\frac{\hats(z)}{\hatp(z)}-\frac{s(z)}{p(z)}\right|^2dz \right)\cdot \left(\int_{\mathbb{T}} 4\|[1,\dots, z^{n-1}]\|^2\cdot |\frac{1}{p(z)}|^2 dz\right) \nonumber
 \le O(n/\tau_1^2) \cdot g(\hatA,\hatC)\mper\nonumber
	 	\end{align}
	Similarly, we could show that $\left\|\partialgrad{g}{\hata}\right\|^2 \le O(n^2/\tau_1^2)\cdot g(\hatA,\hatC)$. 
\end{proof}

\subsection{Justifying idealized risk}

We need to justify the approximation we made in Equation~\eqref{eqn:approximation}. 
\begin{lemma}\label{lem:population_risk}
Assume that $\xi_t$ and $x_t$ are drawn i.i.d. from an arbitrary distribution
with mean~$0$ and variance~$1.$ Then the population risk $f(\hatTheta)$ can be written as, 
\begin{equation}
f(\hatTheta) = (\hatD-D)^2 + \sum_{k=1}^{T-1} \left(1-\frac{k}{T}\right)\left(\hatC \hatA^{k-1}B- CA^{k-1}B\right)^2 + \sigma^2\mper\label{eqn:objective-simplified}
\end{equation}
\end{lemma}
The idealized risk is upper bound of the population risk $f(\hatTheta)$ according to equation~\eqref{eqn:approximation} and~\eqref{eqn:objective-simplified}. We don't have to quantify the gap between them because later in Algorithm~\ref{alg:backprop}, we will directly optimize the idealized risk by constructing an estimator of its gradient, and thus the optimization will guarantee a bounded idealized risk which translates to a bounded population risk. See Section~\ref{sec:learning} for details. 
\begin{proof}[Proof of Lemma~\ref{lem:population_risk}]
Under the distributional assumptions on $\xi_t$ and $x_t,$
we can calculate the objective functions above analytically. We write out $y_t,\haty_t$ in terms of the inputs, 
\begin{align}
y_t = Dx_t + \sum_{k=1}^{t-1} CA^{t-k-1}Bx_{k}  + CA^{t-1}h_0+ \xi_t\,,
\qquad
\haty_t = \hatD x_t + \sum_{k=1}^{t-1} \hatC \hatA^{t-k-1}\hatB x_{k} + CA^{t-1}h_0 \,.\nonumber
\end{align}
Therefore, using the fact that $x_t$'s are independent and with mean 0 and covariance $\Id$, the expectation of the error can be calculated (formally by Claim~\ref{claim:gaussian_expectation}),  \begin{align}
\Exp\left[\|\haty_t-y_t\|^2\right]
& = \|\hatD-D\|_F^2 + \textstyle\sum_{k=1}^{t-1}\big\|\hatC \hatA^{t-k-1}\hatB-CA^{t-k-1}B\big\|_F^2 + \Exp[\|\xi_t\|^2]\mper\label{eqn:37}
\end{align}
Using $\Exp[\|\xi_t\|^2]=\sigma^2\,,$ it follows that 
\begin{equation}
f(\hatTheta) = 
\|\hatD-D\|_F^2 + \textstyle\sum_{k=1}^{T-1} \big(1-\frac{k}{T}\big)\big\|\hatC \hatA^{k-1}\hatB- CA^{k-1}B \big\|_F^2 + \sigma^2\mper \label{eqn:objective-analytical-form}
\end{equation}
Recall that under the controllable canonical form~\eqref{eqn:controllable_form}, $B=e_n$ is known and therefore $\hatB = B$ is no longer a variable. 
Then the expected objective function~\eqref{eqn:objective-analytical-form} simplifies to
	\begin{equation*}
	f(\hatTheta) = (\hatD-D)^2 + \textstyle\sum_{k=1}^{T-1} \big(1-\frac{k}{T}\big)\big(\hatC \hatA^{k-1}B- CA^{k-1}B\big)^2 + \sigma^2\mper
	\end{equation*}
\end{proof}

The previous lemma does not yet control higher order contributions present in the idealized risk. This requires additional structure that we introduce in the next section.

\section{Effective relaxations of spectral radius}
\label{sec:acq}

The previous section showed quasi-convexity of the idealized risk. However, several steps are missing towards showing finite sample guarantees for stochastic gradient descent. In particular, we will need to control the variance of the stochastic gradient at any system that we encounter in the training. For this purpose we formally introduce our main assumption now and show that it serves as an effective relaxation of spectral radius. This results below will be used for proving convergence of stochastic gradient descent in Section~\ref{sec:learning}.

Consider the following convex region~$\ballC$ in the complex plane, 
\begin{equation}
\ballC = \{z\colon \Re z \ge (1+\tau_0)|\Im z|\}
\cap \{z\colon \tau_1 < \Re z <
\tau_2\}\,.\label{eqn:ballc}
\end{equation}
where $\tau_0, \tau_1, \tau_2 > 0$ are constants that are considered as fixed constant throughout the paper. Our bounds will have polynomial dependency on these parameters. Pictorially, this convex set is pretty much the dark area in Figure~\ref{fig:1} (with the corner chopped).  This set in $\C$ induces a convex set in the parameter space which is a subset of the transition matrix with spectral radius less than $\alpha$. 

\begin{definition}
\label{def:strong_passivity}
\label{def:acquiescence}
We say a polynomial $p(z)$ is \emph{$\alpha$-acquiescent} if  $\{p(z)/z^n \colon |z|=\alpha\}\subseteq\ballC.$ A linear system with transfer function $G(z)=s(z)/p(z)$ is $\alpha$-acquiescent if the denominator $p(z)$ is.
\end{definition}

The set of coefficients $a\in\R^n$ defining acquiescent systems form a convex set. Formally, for
a positive $\alpha > 0$, define the convex set $\ball_{\alpha}\subseteq \R^n$ as
\begin{equation}
\ball_{\alpha} = \big\{a\in \R^n: \{p_a(z)/z^n \colon |z|=\alpha\}\subseteq\ballC\big\}\,.\label{eqn:def_ball}
\end{equation}

We note that definition~\eqref{eqn:def_ball} is equivalent to the definition $\ball_{\alpha} = \big\{a\in \R^n: \{z^np(1/z) \colon |z|=1/\alpha\}\subseteq\ballC\big\}$, which is the version that we used in introduction for simplicity.  Indeed, we can verify the convexity of $\ball_{\alpha}$ by definition and the convexity of $\ballC$: $a,b\in B_{\alpha}$ implies that $p_a(z)/z^n, p_b(z)/z^n\in \ballC$ and therefore, $p_{(a+b)/2}(z)/z^n = \frac{1}{2}(p_a(z)/z^n+p_b(z)/z^n)\in \ballC$. We also note that the parameter~$\alpha$ in the definition of acquiescence corresponds to the spectral radius of the companion matrix. In particular, an acquiescent system is stable for $\alpha<1.$ 

\begin{lemma}\label{lem:1}
Suppose $a\in \ball_{\alpha}$, then the roots of polynomial $p_a(z)$ have magnitudes bounded by $\alpha$. 
		Therefore the controllable canonical form $A = \CC(a)$ defined by $a$ has spectral radius $\rho(A) <\alpha$. 
	\end{lemma}

\begin{proof}
	Define holomorphic function $f(z) = z^n$ and $g(z) = p_a(z) = z^n + a_1z^{n-1} + \dots + a_n$. 
	We apply the symmetric form of Rouche's theorem~\cite{estermann1962complex} on the circle $\mathcal{K} = \{z: |z| = \alpha\}$. For any point $z$ on $\mathcal{K}$, we have that $|f(z)| = \alpha^n$,	 and that $|f(z) - g(z)| = \alpha^n \cdot |1- p_a(z)/z^n|$. Since $a\in \ball_{\alpha}$, we have that $p_{a}(z)/z^n\in \ballC$ for any $z$ with $|z| = \alpha$. Observe that for any $c\in \ballC$ we have that $|1-c| < 1+|c|$, therefore we have that 
$$|f(z)-g(z)| = \alpha^n |1-p_a(z)/z^n| < \alpha^n (1+|p_a(z)|/|z^n|) = |f(z)| + |p_a(z)| = |f(z)|+|g(z)|\mper$$
	
	Hence, using Rouche's Theorem, we conclude that $f$ and $g$ have same number of roots inside circle $\mathcal{K}$. Note that function $f=z^n$ has exactly $n$ roots in $\mathcal{K}$ and therefore $g$ have all its $n$ roots inside circle $\mathcal{K}$. 
\end{proof}

The following lemma establishes the fact that $\ball_{\alpha}$ is a monotone family of sets in~$\alpha$.   The proof follows from the maximum modulo principle of the harmonic functions $\Re(z^np(1/z))$ and $\Im(z^np(1/z))$. We defer the short proof to Section~\ref{sec:monotonicity}. We remark that there are larger convex sets than $\ball_{\alpha}$ that ensure bounded spectral radius. However, in order to also guarantee monotonicity and the no blow-up property below, we have to restrict our attention to $\ball_{\alpha}.$

\begin{lemma}[Monotonicity of $\ball_{\alpha}$]\label{lem:monotone}
	For any $0 < \alpha < \beta$, we have that $\ball_{\alpha}\subset \ball_{\beta}$. 
\end{lemma}

Our next lemma entails that acquiescent systems have well behaved impulse responses. 
\begin{lemma}[No blow-up property]\label{lem:no_blowing_up}
 	Suppose $a\in \ball_{\alpha}$ for some $\alpha \le  1$. Then the companion matrix $A = \CC(a)$ satisfies
 	\begin{equation}
 		\sum_{k=0}^{\infty} \|\alpha^{-k}A^kB\|^2 \le 2\pi n\alpha^{-2n}/\tau_1^2.\label{eqn:intermediate2}
 	\end{equation}
 	Moreover, it holds that for any $k\ge 0$, 
 	\begin{equation}
 	\|A^kB\|^2 \le  \min\{2\pi n / \tau_1^2, 2\pi n \alpha^{2k-2n}/\tau_1^2\}\mper\nonumber
 	\end{equation}
 	\end{lemma}
\begin{proof}[Proof of Lemma~\ref{lem:no_blowing_up}]
	
	Let $f_{\lambda} =  \sum_{k=0}^{\infty} e^{i\lambda k}\alpha^{-k}A^kB$ be the Fourier transform of the series $\alpha^{-k}A^kB$. Then using Parseval's Theorem, we have 
	\begin{align}
	\sum_{k=0}^{\infty} \|\alpha^{-k}A^kB\|^2 & = \int_{0}^{2\pi} |f_{\lambda}|^2 d\lambda = \int_{0}^{2\pi} |(I - \alpha^{-1}e^{i\lambda}A)^{-1}B|^2 d\lambda \nonumber\\
	& = \int_{0}^{2\pi} \frac{\sum_{j=1}^n \alpha^{2j}}{|p_a(\alpha e^{-i\lambda})|^2}d\lambda\le \int_0^{2\pi} \frac{n}{|p_a(\alpha e^{-i\lambda})|^2}d\lambda.\label{eqn:intermediate1} 
	\end{align}
	where at the last step we used the fact that $(I - wA)^{-1}B = \frac{1}{p_a(w^{-1})}[
	w^{-1},w^{-2}\dots, z^{-n}]^{\top}$ (see Lemma~\ref{lem:uinverseB}), and that $\alpha \le 1$. 
	Since $a\in \ball_{\alpha}$, we have that $|\phat_a(\alpha^{-1}e^{i\lambda})| \ge \tau_1$, and therefore $p_a(\alpha e^{-i\lambda}) = \alpha^n e^{-in\lambda} \phat(e^{i\lambda}/\alpha)$ has magnitude at least $\tau_1\alpha^n$. Plugging in this into equation~\eqref{eqn:intermediate1}, we conclude that 
	\begin{equation}
	\sum_{k=0}^{\infty} \|\alpha^{-k}A^kB\|^2 \le\int_0^{2\pi} \frac{n}{|p_a(\alpha e^{-i\lambda})|^2}d\lambda\le 2\pi n\alpha^{-2n}/\tau_1^2.\nonumber
	\end{equation}	
	Finally we establish the bound for $\|A^kB\|^2$. 	By Lemma~\ref{lem:monotone}, we have $\ball_{\alpha} \subset \ball_{1}$ for $\alpha \le 1$. Therefore we can pick $\alpha = 1$ in equation~\eqref{eqn:intermediate2} and it still holds. That is, we have that 
	\begin{equation}
	\sum_{k=0}^{\infty} \|A^kB\|^2 \le 2\pi n /\tau_1^2.\nonumber
	\end{equation}	
	This also implies that $\|A^kB\|^2 \le 2\pi n/\tau_1^2$. 
\end{proof}

\subsection{Efficiently computing the projection}

In our algorithm, we require a projection onto $\mathcal{B}_\alpha$.  However, the only requirement of the projection step is that it projects onto a set contained inside $\mathcal{B}_\alpha$ that  also contains the true linear system.  So a variety of subroutines can be used to compute this projection or an approximation.  First, the explicit projection onto $\mathcal{B}_\alpha$ is representable by a semidefinite program.  This is because each of the three constrains can be checked by testing if a trigonometric polynomial is non-negative.  A simple inner approximation can be constructed by requiring the constraints to hold on an a finite grid of size $O(n)$. One can check that this provides a tight, polyhedral approximation to the set $\mathcal{B}_\alpha$, following an argument similar to Appendix C of Bhaskar~\emph{et al}~\cite{Bhaskar13}. Projection to this polyhedral takes at most $O(n^{3.5})$ time by linear programming and potentially can be made faster by using fast Fourier transform. See Section~\ref{sec:projection} for more detailed discussion on why projection on a polytope suffices. Furthermore, sometimes we can replace the constraint by an $\ell_1$ or $\ell_2$-constraint if we know that the system satisfies the corresponding assumption. Removing the projection step entirely is an interesting open problem.

\section{Learning acquiescent systems}
\label{sec:learning}
\label{sec:proper}
Next we show that we can learn acquiescent systems.
\begin{theorem}\label{thm:stong_passivity}
		Suppose the true system $\Theta$ is $\alpha$-acquiescent and satisfies $\|C\|\le 1$. Then with $N$ samples of length  $T\ge \Omega(n+1/(1-\alpha))$, stochastic gradient descent (Algorithm~\ref{alg:backprop}) with projection set $\ball_{\alpha}$ returns parameters $\hatTheta = (\hatA,\hatB,\hatC,\hatD)$ with population risk 
	\begin{equation}
	\Exp f(\hatTheta) \le f(\Theta) + O\left(\frac{n^2}{N} +\sqrt{\frac{n^5+\sigma^2n^3}{TN}}\right) \,,	\label{eqn:passivity_theorem}
	\end{equation}
		where $O(\cdot)$-notation hides polynomial dependencies on $1/(1-\alpha)$, $1/\tau_0,1/\tau_1,\tau_2$, and $R = \|a\|$. The expectation is taken over the randomness of the algorithms and  the examples.  
\end{theorem}

\begin{algorithm}\caption{Projected stochastic gradient descent with partial loss}\label{alg:backprop}
	\vspace{.1in}
	\textbf{For} $i = 0$ to $N$:	
	\begin{enumerate}
		\item Take a fresh sample $((x_1,\dots, x_T),(y_1,\dots, y_T))$. Let $\tilde{y}_t$ be the simulated outputs\footnotemark{} of system $\hatTheta$ on inputs $x$ and initial states $h_0 = 0$. 
	 	\item Let $T_1 = T/4$. Run stochastic gradient descent\footnotemark{} on loss function $\ell((x,y),\hatTheta) = \frac{1}{T-T_1}\sum_{t >T_1}\|\tilde{y}_t-y_{t}\|^2$. 	 	Concretely, let $G_A = \partialgrad{\ell}{\hata}$, $G_C = \partialgrad{\ell}{\hatC}$, and , $G_D = \partialgrad{\ell}{\hatD}$, we update 
		\begin{equation}
		[\hata,\hatC,\hatD]\rightarrow 	[\hata,\hatC,\hatD] - \eta [G_A,G_C,G_D]\mper\nonumber
		\end{equation}
		\item Project $\hatTheta = (\hata,\hatC,\hatD)$ to the set $\ball_{\alpha}\otimes \R^n\otimes \R$. 
	\end{enumerate} 
		\end{algorithm}
\footnotetext{Note that $\tilde{y}_t$ is different from $\hat{y}_t$ defined in equation~\eqref{eqn:haty_t} which is used to define the population risk: here $\hat{y}_t$ is obtained from the (wrong) initial state $h_0 = 0$ while $\hat{y}_t$ is obtained from the correct initial state.}  
\footnotetext{See Algorithm Box~\ref{alg:back_prop_general} for a detailed back-propagation algorithm that computes the gradient. }

Recall that $T$ is the length of the sequence and $N$ is the number of samples. The first term in the bound~\eqref{eqn:passivity_theorem} comes from the smoothness of the population risk and the second comes from the variance of the gradient estimator of population risk (which will be described in detail below). An important (but not surprising) feature here is the variance scale in $1/T$ and therefore for long sequence actually we got $1/N$ convergence instead of $1/\sqrt{N}$ (for relatively small $N$).

\vspace{.05in}

\noindent \textit{Computational complexity: } Step 2 in each iteration of the algorithm takes $O(Tn)$ arithmetic operations, and the projection step takes $O(n^{3.5})$ time to solve an linear programming problem.  The project step is unlikely to be required in practice and may be an artifact of our analysis.

We can further balance the variance of the estimator with the number of samples by breaking each long sequence of length $T$ into $\Theta(T/n)$ short sequences of length $\Theta(n)$, and then run back-propagation~\eqref{alg:backprop} on these $TN/n$ shorter sequences. This leads us to the following bound which gives the right dependency in $T$ and $N$ as we expected: $TN$ should be counted as the true number of samples for the sequence-to-sequence model.

\begin{corollary}\label{cor:main}
	Under the assumption of Theorem~\ref{thm:stong_passivity}, Algorithm~\ref{alg:cor} returns parameters~$\hatTheta$ with population risk 
		\begin{equation}
\Exp	f(\hatTheta) \le f(\Theta) + O\left(\sqrt{\frac{n^5+\sigma^2n^3}{TN}}\right)\,,\nonumber
	\end{equation}
	where $O(\cdot)$-notation hides polynomial dependencies on $1/(1-\alpha)$, $1/\tau_0,1/\tau_1,\tau_2$, and $R = \|a\|.$ 
		\end{corollary}

\begin{algorithm}\caption{Projected stochastic gradient descent for long sequences}\label{alg:cor}
	\textbf{Input:} $N$ samples sequences of length $T$\\
	\textbf{Output:} Learned system $\hatTheta$
	\begin{enumerate}
		\item Divide each sample of length $T$ into $T/(\beta n)$ samples of length $\beta n$ where $\beta$ is a large enough constant. Then run algorithm~\ref{alg:backprop} with the new samples and obtain $\hatTheta$. 	\end{enumerate}
\end{algorithm}

We remark the the gradient computation procedure takes time linear in $Tn$ since one can use chain-rule (also called back-propagation) to compute the gradient efficiently . For completeness, Algorithm~\ref{alg:back_prop_general} gives a detailed implementation. 
Finally and importantly, we remark that although we defined the population risk as the expected error with respected to sequence of length $T$, actually our error bound generalizes to any longer (or shorter) sequences of length $T' \gg \max\{n,1/{(1-\alpha)}\}$. By the explicit formula for $f(\hatTheta)$ (Lemma~\ref{lem:population_risk}) and the fact that $\|CA^kB\|$ decays exponentially for $k \gg n$ (Lemma~\ref{lem:no_blowing_up}), we can bound the population risk on sequences of different lengths. Concretely,  let $f_{T'}(\hatTheta)$ denote the population risk on sequence of length $T'$, we have for all $T'\gg \max\{n,1/{(1-\alpha)}\}$, 
\begin{equation}
f_{T'}(\hatTheta) \le 1.1f(\hatTheta) + \exp(-{(1-\alpha)} \min\{T,T'\})\le O\left(\sqrt{\frac{n^5+\sigma^2n^3}{TN}}\right)\mper\nonumber
\end{equation}

We note that generalization to longer sequence does deserve attention. Indeed in practice, it's usually difficult to train non-linear recurrent networks that generalize to longer sequences than the training data.

We could hope to achieve linear convergence by showing that the empirical risk also satisfies the weakly-quasi-convexity. Then, we can re-use the samples and hope to use strong optimization tools (such as SVRG) to achieve the linear convergence. This is beyond the scope of this paper and left to future work. 

Our proof of Theorem~\ref{thm:stong_passivity} simply consists of three parts: a) showing the idealized risk is weakly quasi-convex in the convex set $\ball_{\alpha}$ (Lemma~\ref{lem:convexity_smoothness}); b) designing an (almost) unbiased estimator of the gradient of the idealized risk (Lemma~\ref{lem:unbiased_estimator}); c) variance bounds of the gradient estimator (Lemma~\ref{lem:variance}). 

First of all, using the theory developed in Section~\ref{sec:fourier} (Lemma~\ref{lem:main} and Lemma~\ref{lem:smooth}), it is straightforward to verify that in the convex set $\ball_{\alpha}\otimes \R^n$, the idealized risk is both weakly-quasi-convex and weakly-smooth. 
\begin{lemma}\label{lem:convexity_smoothness}
	Under the condition of Theorem~\ref{thm:stong_passivity}, the idealized risk~\eqref{eqn:ideal-objective} is $\tau$-weakly-quasi-convex in the convex set $\ball_{\alpha} \otimes \R^n$ and $\Gamma$-weakly smooth, where $\tau = \Omega(\tau_0\tau_1/\tau_2)$ and $\Gamma = O(n^2/\tau_1^4)$. 
\end{lemma}

\begin{proof}[Proof of Lemma~\ref{lem:convexity_smoothness}]
	It suffices to show that for all $\hata, a\in \ball_{\alpha}$, it satisfies $\hata \in \neighbor(a)$ for $\tau = \Omega(\tau_0\tau_1/\tau_2)$. Indeed, by the monotonicity of the family of sets $\ball_{\alpha}$ (Lemma~\ref{lem:monotone}), we have that  $\hata,a\in \ball_{1}$, which by definition means for every $z$ on unit circle, $p_a(z)/z^n, p_{\hata}(z)/z^n\in \ballC$. By definition of $\ballC$, for any point $w,\hat{w}\in \ballC$, the angle $\phi$ between $w$ and $\hat{w}$ is at most $\pi-\Omega(\tau_0)$ and ratio of the magnitude is at least $\tau_1/\tau_2$, which implies that $\Re(w/\hat{w}) = |w|/|\hat{w}|\cdot \cos(\phi)\ge \Omega(\tau_0\tau_1/\tau_2)$. Therefore $\Re(p_a(z)/p_{\hata}(z))\ge \Omega(\tau_0\tau_1/\tau_2)$, and we conclude that $\hata\in \neighbor(a)$. The smoothness bound was established in Lemma~\ref{lem:smooth}. 
\end{proof}

Towards designing an unbiased estimator of the gradient, we note that there is a small caveat here that prevents us to just use the gradient of the empirical risk, as commonly done for other (static) problems. 
Recall that the population risk is defined as the expected risk with \textit{known} initial state $h_0$, while in the training we don't have access to the initial states and therefore using the naive approach we couldn't even estimate population risk from samples without knowing the initial states. 

We argue that being able to handle the missing initial states is indeed desired: in most of the interesting applications $h_0$ is  unknown (or even to be learned). Moreover, the ability of handling unknown $h_0$ allows us to break a long sequence into shorter sequences, which helps us to obtain Corollary~\ref{cor:main}. Here the difficulty is essentially that we have a supervised learning problem with missing data $h_0$. We get around it by simply ignoring
first $T_1 = \Omega(T)$ outputs of the system and setting the corresponding errors to 0. Since the influence of $h_0$ to any outputs later than time $k \ge T_1 \gg  \max\{n,1/{(1-\alpha)}\}$ is inverse exponentially small, we could safely assume $h_0 = 0$ when the error earlier than time $T_1$ is not taken into account. 

This small trick also makes our algorithm suitable to the cases when these early outputs are actually not observed. This is indeed an interesting setting, since in many sequence-to-sequence model~\cite{DBLP:conf/nips/SutskeverVL14}, there is no output in the first half fraction of iterations (of course these models have non-linear operation that we cannot handle).

The proof of the correctness of the estimator is almost trivial and deferred to Section~\ref{sec:proofs}.

\begin{lemma}\label{lem:unbiased_estimator}
		Under the assumption of Theorem~\ref{thm:stong_passivity}, suppose $\hata,a\in \ball_{\alpha}$. Then in Algorithm~\ref{alg:backprop}, at each iteration, $G_A,G_C$ are unbiased estimators of the gradient of the idealized risk~\eqref{eqn:ideal-objective} in the sense that:
	\begin{align}
\Exp\left[G_A,G_C\right] & =  \left[\partialgrad{g}{\hata},\partialgrad{g}{\hatC}\right]\pm \exp(-\Omega({(1-\alpha)} T))\mper\nonumber	\\	\end{align} 
\end{lemma}

Finally, we control the variance of the gradient estimator. 

\begin{lemma}\label{lem:variance}
		The (almost) unbiased estimator $(G_A,G_C)$ of the gradient of $g(\hatA,\hatC)$ has variance bounded by 
	\begin{equation}
	\Var\left[G_{A}\right] + \Var\left[G_{C}\right]\le \frac{O\left(n^3\Lambda^2/\tau_1^6+ \sigma^2n^2\Lambda/\tau_1^4\right)}{T} \mper\nonumber
	\end{equation}
	where $\Lambda = O(\max\{n,1/{(1-\alpha)} \log 1/{(1-\alpha)}\})$. 
				
			\end{lemma}

Note that Lemma~\ref{lem:variance} does not directly follow from the $\Gamma$-weakly-smoothness of the population risk, since it's not clear whether the loss function $\ell((x,y),\hatTheta)$  is also $\Gamma$-smooth for every sample. Moreover, even if it could work out, from smoothness the variance bound can be only as small as $\Gamma^2$, while the true variance scales linearly in $1/T$. Here the discrepancy comes from that smoothness implies an upper bound of the expected squared norm of the gradient, which is equal to the variance plus the expected squared mean.  Though typically for many other problems variance is on the same order as the squared mean, here for our sequence-to-sequence model, 
actually the variance decreases in length of the data, and therefore the bound of variance from smoothness is pessimistic. 

We bound directly the variance instead. It's tedious but simple in spirit. We mainly need Lemma~\ref{lem:no_blowing_up} to control various difference sums that shows up from calculating the expectation. The only tricky part is to obtain the $1/T$ dependency which corresponds to the cancellation of the contribution from the cross terms. In the proof we will basically write out the variance as a (complicated) function of $\hatA,\hatC$ which consists of sums of terms involving $(\hatC\hatA^k B-CA^kB)$ and $\hatA^kB$. We control these sums using Lemma~\ref{lem:no_blowing_up}. The proof is deferred to Section~\ref{sec:proofs}.

Finally we are ready to prove Theorem~\ref{thm:stong_passivity}. We essentially just combine Lemma~\ref{lem:convexity_smoothness}, Lemma~\ref{lem:unbiased_estimator} and Lemma~\ref{lem:variance} with the generic convergence Proposition~\ref{lem:framework}. This will give us low error in idealized risk and then we relate the idealized risk to the population risk. 

\begin{proof}[Proof of Theorem~\ref{thm:stong_passivity}]
We consider $g'(\hatA,\hatC,\hatD)= (\hatD-D)^2 + g(\hatA,\hatC)$, an extended version of the idealized risk which takes the contribution of $\hatD$ into account. By Lemma~\ref{lem:unbiased_estimator} we have that Algorithm~\ref{alg:backprop} computes $G_A,G_C$ which are almost unbiased estimators of the gradients of $g'$ up to negligible error $\exp(-\Omega({(1-\alpha)} T))$, and by Lemma~\ref{lem:unbisaed_estimator_general} we have $G_D$ is an unbiased estimator of $g'$ with respect to $\hatD$.  Moreover by Lemma~\ref{lem:variance}, these unbiased estimator has total variance $V = \frac{O\left(n^5+ \sigma^2n^3\right)}{T}$ where $O(\cdot)$ hides dependency on $\tau_1$ and ${(1-\alpha)}$. Applying Proposition~\ref{lem:framework} (which only requires an unbiased estimator of the gradient of $g'$),  we obtain that after $T$ iterations, we converge to a point with $g'(\hata,\hatC,\hatD)\le  O\left(\frac{n^2}{N} +\sqrt{\frac{n^5+\sigma^2n^3}{TN}}\right)$. Then, by Lemma~\ref{lem:population_risk} we have 
	$f(\hatTheta) \le g'(\hata,\hatC,\hatD)+ \sigma^2 = g'(\hata,\hatC,\hatD) + f(\Theta)\le O\left(\frac{n^2}{N} +\sqrt{\frac{n^5+\sigma^2n^3}{TN}}\right) + f(\Theta)$ which completes the proof. 
\end{proof}

\section{The power of improper learning}\label{sec:improper}
\newcommand{\acext}[2]{$#1$-acquiescent by extension of degree $#2$}
\newcommand{\acextshort}{acquiescent by extension}
We observe an interesting and important fact about the theory in Section~\ref{sec:learning}: it solely requires a condition on the characteristic function $p(z)$. This suggests that the geometry of the training objective function depends mostly on the denominator of the transfer function, even though the system is uniquely determined by the transfer function $G(z) = s(z)/p(z)$. This might seem to be an undesirable discrepancy between the behavior of the system and our analysis of the optimization problem.

However, we can actually exploit this discrepancy to design improper learning algorithms that succeed under much weaker assumptions. We rely on the following simple observation about the invariance of a system $G(z) = \frac{s(z)}{p(z)}$. For an arbitrary polynomial $u(z)$ of leading coefficient 1, we can write $G(z)$ as 
$$G(z) = \frac{s(z)u(z)}{p(z)u(z)} = \frac{\tilde{s}(z)}{\tilde{p}(z)}\,,$$
where $\tilde{s} = su$ and $\tilde{p} = pu$. 
Therefore the system $\tilde{s}(z)/\tilde{p}(z)$ has identical behavior as $G$. Although this is a redundant representation of $G(z)$, it should counted as an acceptable solution. After all, learning the minimum representation\footnote{The minimum representation of a transfer function $G(z)$ is defined as the representation $G(z) = s(z)/p(z)$ with $p(z)$ having minimum degree.} of linear system is impossible in general. In fact, we will encounter an example in Section~\ref{subsec:min_representation}.

While not changing the behavior of the system, the extension from $p(z)$ to $\tilde{p}(z),$ does affect the geometry of the optimization problem. In particular, if $\tilde{p}(z)$ is now an $\alpha$-acquiescent characteristic polynomial as defined in Definition~\ref{def:strong_passivity}, then we could find it simply using stochastic gradient descent as shown in Section~\ref{sec:learning}. 
Observe that we don't require knowledge of $u(z)$ but only its existence. Denoting by $d$ the degree of $u,$ the algorithm itself is simply stochastic gradient descent with $n+d$ model parameters instead of $n.$

Our discussion motivates the following definition.
\begin{definition}\label{def:nice_extensible}
	A polynomial $p(z)$ of degree $n$ is \emph{\acext{\alpha}{d}} if there exists a polynomial $u(z)$ of degree $d$ and leading coefficient 1 such that $p(z)u(z)$ is $\alpha$-acquiescent.
\end{definition}

For a transfer function $G(z)$, we define it's $\mathcal{H}_2$ norm as 

\begin{equation}
\|G\|_{\mathcal{H}_2}^2  =\frac{1}{2\pi}\int_{0}^{2\pi} |G(e^{i\theta})|^2 d\theta\mper\nonumber
\end{equation}

We assume  (with loss of generality) that the true transfer function $G(z)$ has bounded $\mathcal{H}_2$ norm, that is, $\|G\|_{\mathcal{H}_2}\le 1$. This can be achieve by a rescaling\footnote{In fact, this is a natural scaling that makes comparing error easier.  Recall that the population risk is essentially $\|\hat{G}-G\|_{\mathcal{H}_2}$, therefore rescaling $C$ so that $\|G\|_{\mathcal{H}_2} =1$ implies that when error $\ll 1$ we achieve non-trivial performance. } of the matrix $C$. 

\begin{theorem}\label{thm:improper}
	Suppose the true system has transfer function $G(z)= s(z)/p(z)$ with a characteristic function $p(z)$ that is~\acext{\alpha}{d}, and $\|G\|_{\mathcal{H}_2}\le 1$, then projected stochastic gradient descent with $m = n+d$ states (that is, Algorithm~\ref{alg:cor} with $m$ states) returns a system 
		$\hatTheta$ with population risk 
	\begin{equation}
	f(\hatTheta) \le O\left(\sqrt{\frac{m^5+\sigma^2m^3}{TK}}\right)\mper\nonumber
	\end{equation}
	where the $O(\cdot)$ notation hides polynomial dependencies on $\tau_0,\tau_1,\tau_2, 1/(1-\alpha)$. 
\end{theorem}

The theorem follows directly from Corollary~\ref{cor:main} (with some additional care about the scaling.

\begin{proof}[Proof of Theorem~\ref{thm:improper}]
	Let $\tilde{p}(z) = p(z)u(z)$ be the acquiescent extension of $p(z)$. Since $\tau_2\ge |u(z)p(z)| = |\tilde{p}(z)|\ge \tau_0$ on the unit circle, we have that $|\tilde{s}(z)| = |s(z)||u(z)| = s(z)\cdot O_{\tau}(1/p(z))$. Therefore we have that 
	 $\tilde{s}(z)$ satisfies that $\|\tilde{s}\|_{\mathcal{H}_2} = O_{\tau}(\|s(z)/p(z)\|_{\mathcal{H}_2}) = O_{\tau}(\|G(z)\|_{\mathcal{H}_2}) \le O_{\tau}(1)$. That means that the vector $C$ that determines the coefficients of $\tilde{s}$ satisfies that $\|C\|\le O_{\tau}(1),$ since for a polynomial $h(z) = b_0+\dots + b_{n-1}z^{n-1}$, we have $\|h\|_{\mathcal{H}_2} = \|b\|.$ Therefore we can apply Corollary~\ref{cor:main} to complete the proof.
	 \end{proof}

In the rest of this section, we discuss in subsection~\ref{subsec:min_representation} the instability of the minimum representation in subsection,  and in subsection~\ref{subsec:improper_examples} we show several examples where the characteristic function $p(z)$ is not $\alpha$-acquiescent but is $\alpha$-acquiescent by extension with small degree $d$. 

As a final remark, the examples illustrated in the following sub-sections may be far from optimally analyzed. It is beyond the scope of this paper to understand the optimal condition under which $p(z)$ is \acextshort.  

\subsection{Instability of the minimum representation} 
\label{subsec:min_representation}
We begin by constructing a contrived example where the minimum representation of $G(z)$ is not stable at all and as a consequence one can't hope to recover the minimum representation of $G(z)$. 

Consider $G(z) = \frac{s(z)}{p(z)} := \frac{z^n - 0.8^{-n}}{(z-0.1)(z^n-0.9^{-n})}$ and $G'(z) =\frac{s'(z)}{p'(z)} := \frac{1}{z-0.1}$. Clearly these are the minimum representations of the $G(z)$ and $G'(z)$, which also both satisfy acquiescence. On the one hand, the characteristic polynomial $p(z)$ and $p'(z)$ are very different. On the other hand, the transfer functions $G(z)$ and $G'(z)$ have almost the same values on unit circle up to exponentially small error, 
\begin{equation}
|G(z) - G'(z)|\le \frac{0.8^{-n}-0.9^{-n}}{(z-0.1)(z-0.9^{-n})} \le \exp(-\Omega(n))\mper\nonumber
\end{equation}
Moreover, the transfer functions $G(z)$ and $\hat{G}(z)$ are on the order of $\Theta(1)$ on unit circle. These suggest that from an (inverse polynomially accurate) approximation of the transfer function $G(z)$, we cannot hope to recover the minimum representation in any sense, even if the minimum representation satisfies acquiescence.

\subsection{Power of improper learning in various cases}
\label{subsec:improper_examples}

We illustrate the use of improper learning through various examples below.

\subsubsection{Example: artificial construction}

We consider a simple contrived example where improper learning can help us learn the transfer function dramatically. We will show an example of characteristic function which is not $1$-acquiescent but $(\alpha+1)/2$-\acext{(\alpha+1)/2}{3}. 

Let $n$ be a large enough integer and $\alpha$ be a constant. Let $J = \{1,n-1,n\}$ and $\omega = e^{2\pi i/n}$, and then define $p(z) = z^3\prod_{j\in [n], j\notin J} (z-\alpha\omega^j)$. Therefore we have that \begin{align}p(z)/z^n =  z^3\prod_{j\in [n],j\in J}(1-\alpha\omega^j/z) = \frac{1-\alpha^n/z^n}{(1-\omega/z)(1-\omega^{-1} /z)(1-1/z)}\label{eqn:eqn29}\end{align}

Taking $z = e^{-i\pi/2}$ we have that $p(z)/z^n$ has argument (phase) roughly $-3\pi/4$, and therefore it's not in $\ballC$, which implies that $p(z)$ is \textit{not} $1$-acquiescent. On the other hand, picking $u(z) = (z-\omega)(z-1)(z-\omega^{-1})$ as the helper function, from equation~\eqref{eqn:eqn29}  we have $p(z)u(z)/z^{n+3} = 1-\alpha^n/z^n$ takes values inverse exponentially close to $1$ on the circle with radius $(\alpha+1)/2$. Therefore $p(z)u(z)$ is $(\alpha+1)/2$-acquiescent. 

\subsubsection{Example: characteristic function with separated roots}\label{subsec:separate_roots}

A characteristic polynomial with well separated roots will be \acextshort. Our bound will depend on the following quantity of $p$ that characterizes the separateness of the roots.
\begin{definition}
	For a polynomial $h(z)$ of degree $n$ with roots $\lambda_1,\dots,\lambda_{n}$ inside unit circle,  define the quantity $\Gamma(\cdot)$ of the polynomial $h$ as: 
	\begin{align*}
	\\	\Gamma(h) &: = \sum_{j\in [n]} \left|\frac{\lambda_j^{n}}{\prod_{i\neq j}(\lambda_i-\lambda_j)}\right|\mper 
		\end{align*}
\end{definition}
\begin{lemma}\label{lem:distinct_roots}
	Suppose $p(z)$ is a polynomial of degree $n$ with distinct roots inside circle with radius $\alpha$. 	Let $\Gamma=\Gamma(p)$, then $p(z)$ is \acext{\alpha}{d}		$~= O(\max\{(1-\alpha)^{-1}\log (\sqrt{n}\Gamma\cdot \|p\|_{\mathcal{H}_2}), 0\})$. 						\end{lemma}

Our main idea to extend $p(z)$ by multiplying some polynomial $u$ that approximates $p^{-1}$ (in a relatively weak sense) and therefore $pu$ will always take values in the set $\ballC$. We believe the following lemma should be known though for completeness we provide the proof in Section~\ref{sec:improper_appendix}. 

\begin{lemma}[Approximation of inverse of a polynomial]\label{lem:approximation_inverse}
	Suppose $p(z)$ is a polynomial of degree $n$ and leading coefficient 1 with distinct roots inside circle with radius $\alpha$, and $\Gamma=\Gamma(p)$. 
		Then for $d = O(\max\{(\frac{1}{1-\alpha}\log\frac{\Gamma}{(1-\alpha)\zeta}, 0\})$, 
	there exists a polynomial $h(z)$ of degree $d$ and leading coefficient 1 such that for all $z$ on unit circle, 
	\begin{equation}
	\left|\frac{z^{n+d}}{p(z)} - h(z)\right| \le \zeta \mper\nonumber
	\end{equation}
\end{lemma}

\begin{proof}[Proof of Lemma~\ref{lem:distinct_roots}]
	Let $\gamma = 1-\alpha$. Using Lemma~\ref{lem:approximation_inverse} with $\zeta = 0.5\|p\|_{\mathcal{H}_{\infty}}^{-1}$, we have that there exists polynomial $u$ of degree $d = O(\max\{\frac{1}{1-\alpha}\log(\Gamma\|p\|_{\mathcal{H}_{\infty}}),0\})$ such that 
	\begin{equation}
	\left|\frac{z^{n+d}}{p(z)} - u(z)\right| \le \zeta \mper\nonumber
	\end{equation}
	Then we have that 
	$$\left|p(z)u(z)/z^{n+d} -1\right|\le \zeta|p(z)| < 0.5\mper $$
	Therefore $p(z)u(z)/z^{n+d} \in \ballC_{\tau_0,\tau_1,\tau_2}$ for constant $\tau_0,\tau_1,\tau_2$. Finally noting that for degree $n$ polynomial we have $\|h\|_{\mathcal{H}_{\infty}}\le \sqrt{n}\cdot \|h\|_{\mathcal{H}_{2}}$, which completes the proof. \end{proof}

\subsubsection{Example: Characteristic polynomial with random roots}

We consider the following generative model for characteristic polynomial of degree $2n$. We generate $n$ complex numbers $\lambda_1,\dots, \lambda_n$ uniformly randomly on circle with radius $\alpha < 1$, and take $\lambda_i, \bar{\lambda}_i$ for $i = 1,\dots, n$ as the roots of $p(z)$. That is, $p(z) = (z-\lambda_1)(z-\bar{\lambda}_1)\dots (z-\lambda_n)(z-\bar{\lambda}_n)$. We show that with good probability (over the randomness of $\lambda_i$'s), polynomial $p(z)$ will satisfy the condition in subsection~\ref{subsec:separate_roots} so that it can be learned efficiently by our improper learning algorithm. 

\begin{theorem} \label{thm:random_roots}Suppose $p(z)$ with random roots inside circle of radius $\alpha$ is generated from the process described above. Then with high probability over the choice of $p$, we have that $\Gamma(p)\le \exp(\widetilde{O}(\sqrt{n}))$ and $\|p\|_{\mathcal{H}_2}\le \exp(\tildO{\sqrt{n}})$.  As a corollary, $p(z)$ is~\acext{\alpha}{\widetilde{O}((1-\alpha)^{-1}n)}. 
\end{theorem}

Towards proving Theorem~\ref{thm:random_roots}, we need the following lemma about the expected distance of two random points with radius $\rho$ and $r$ in log-space. 
\begin{lemma}\label{lem:exp_log}
	Let $x\in \C$ be a fixed point with $|x| =\rho$, and $\lambda$ uniformly drawn on the circle with radius $r$. Then 
		$\Exp\left[\ln|x-\lambda|\right] = \ln \max\{\rho,r\}\mper\nonumber$
	\end{lemma}

\begin{proof}
	When $r\neq \rho$, let $N$ be an integer and $\omega = e^{2i\pi/N}$. Then we have that
	\begin{equation}
	\Exp[\ln |x-\lambda|\mid r] = \lim_{N\rightarrow \infty} \frac{1}{N} \sum_{k=1}^N \ln |x-r\omega^k|\label{eqn:eqn28}
	\end{equation}
	The right hand of equation~\eqref{eqn:eqn28} can be computed easily by observing that $\frac{1}{N} \sum_{k=1}^N \ln |x-r\omega^k| = \frac{1}{N} \ln \left|\prod_{k=1}^N (x-r\omega^k)\right|  = \frac{1}{N} \ln |x^N-r^N|$. Therefore, when $\rho > r$, we have $\lim_{N\rightarrow \infty} \frac{1}{N} \sum_{k=1}^N \ln |x-r\omega^k| = \lim_{N\rightarrow \infty}\rho + \frac{1}{N} \ln |(x/\rho)^N- (r/\rho)^N| = \ln \rho$. On the other hand, when $\rho < r$, we have that $\lim_{N\rightarrow \infty} \frac{1}{N} \sum_{k=1}^N \ln |x-r\omega^k| = \ln r$. Therefore we have that $\Exp[\ln |x-\lambda|\mid r]  = \ln(\max{\rho,r})$. 
	For $\rho  = r$, similarly proof (with more careful concern of regularity condition) we can show that $\Exp[\ln |x-\lambda|\mid r]  = \ln r$. 
\end{proof}

Now we are ready to prove Theorem~\ref{thm:random_roots}. 

\begin{proof}[Proof of Theorem~\ref{thm:random_roots}]
	Fixing index $i$, and the choice of $\lambda_i$, 	we consider the random variable $Y_i = \ln(\frac{|\lambda_i|^{2n}}{\prod_{j\neq i}|\lambda_i-\lambda_j|\prod_{j\neq i}|\lambda_i-\bar{\lambda_j}|})n\ln |\lambda_i|  - \sum_{j\neq i}\ln|\lambda_i-\lambda_j|$. By Lemma~\ref{lem:exp_log}, we have that  $\Exp[Y_i] = n\ln |\lambda_i| - \sum_{j\neq i}\Exp[\ln|\lambda_i-\lambda_j|]= \ln (1-\delta). $ Let $Z_j = \ln|\lambda_i-\lambda_j|$. Then we have that $Z_j$ are random variable  with mean 0 and $\psi_1$-Orlicz norm bounded by $1$ since $\Exp[e^{\ln|\lambda_i-\lambda_j|} -1 ]\le 1$. Therefore by Bernstein inequality for sub-exponential tail random variable (for example, ~\cite[Theorem 6.21]{ledoux2013probability}), we have that with high probability ($1-n^{-10}$), 
	it holds that $\left|\sum_{j\neq i}Z_j\right|\le \widetilde{O}(\sqrt{n})$ where $\widetilde{O}$ hides logarithmic factors. Therefore, with high probability, we have $|Y_i|\le \widetilde{O}(\sqrt{n})$. 
	
	Finally we take union bound over all $i\in [n]$, and obtain that with high probability, for $\forall i\in [n], |Y_i|\le \widetilde{O}(\sqrt{n})$, which implies that $\sum_{i=1}^n\exp(Y_i)\le \exp(\widetilde{O}(\sqrt{n}))$. With similar technique, we can prove that $\|p\|_{\mathcal{H}_2}\le \exp(\tilde{O}(\sqrt{n})$. 
	\end{proof}

\subsubsection{Example: Passive systems}\label{subsec:passive_improper}
We will show that with improper learning we can learn almost all passive systems, an important class of stable linear dynamical system as we discussed earlier. We start off with the definition of a strict-input passive system. 
\begin{definition}[Passive System, c.f~\cite{kottenstette2010relationships}]
	A SISO linear system is strict-input passive if and only if for some $\tau_0 > 0$ and any $z$ on unit circle, 
	$\Re(G(z)) \ge \tau_0\mper$
\end{definition}

In order to learn the passive system, we need to add assumptions in the definition of strict passivity. To make it precise, we define the following subsets of complex plane: For positive constant $\tau_0,\tau_1,\tau_2$, define
\begin{equation}
\ballC^+_{\tau_0,\tau_1,\tau_2} = \{z\in \C: |z| \le \tau_2, \Re(z)\ge \tau_1, \Re(z)\ge \tau_0|\Im(z)|~\}\mper
\end{equation}

We say a transfer function $G(z) = s(z)/p(z)$ is $(\tau_0,\tau_1,\tau_2)$-strict input passive if for any $z$ on unit circle we have $G(z)\in \ballC^+_{\tau_0,\tau_1,\tau_2} $. Note that for small constant $\tau_0,\tau_1$ and large constant $\tau_2$, this basically means the system is strict-input passive. 

Now we are ready to state our main theorem in this subsection. We will prove that  passive systems could be learned improperly with a constant factor more states (dimensions), assuming $s(z)$ has all its roots strictly inside unit circles and $\Gamma(s) \le \exp(O(n))$. 
\begin{theorem}\label{thm:improper_passive}
	Suppose $G(z) = s(z)/p(z)$ is $(\tau_0,\tau_1,\tau_2)$-strict-input passive. Moreover, suppose the roots of $s(z)$ have magnitudes inside circle with radius $\alpha$ and $\Gamma=\Gamma(s)\le \exp(O(n))$ and $\|p\|_{\mathcal{H}_2}\le \exp(O(n))$.  Then $p(z)$ is \acext{\alpha}{d= O_{\tau,\alpha}(n)}, and as a consequence we can learn $G(z)$ with $n+d$ states in polynomial time. 
 
 Moreover, suppose in addition we assume that $G(z)\in \ballC_{\tau_0,\tau_1,\tau_2}$ for every $z$ on unit circle. Then $p(z)$ is \acext{\alpha}{d= O_{\tau,\alpha}(n)}. 
\end{theorem}

The proof of Theorem~\ref{thm:improper_passive} is similar in spirit to that of Lemma~\ref{lem:distinct_roots}, and is deferred to Section~\ref{sec:improper_appendix}. 

\subsection{Improper learning using linear regression}\label{subsec:linear_regression}

In this subsection, we show that under stronger assumption than $\alpha$-acquiescent by extension, we can improperly learn a linear dynamical system with linear regression, up to some fixed bias. 

The basic idea is to fit a linear function that maps $[x_{k-\ell},\dots, x_{k}]$ to $y_{k}$. This is equivalent to a dynamical system with $\ell$ hidden states and with the companion matrix $A$ in~\eqref{eqn:controllable_form} being chosen as $a_{\ell} = 1$ and $a_{\ell-1} =\dots = a_1 = 0$. In this case, the hidden states exactly memorize all the previous $\ell$ inputs, and the output is a linear combination of the hidden states. 

Equivalently, in the frequency space, this corresponds to fitting the transfer function $G(z) = s(z)/p(z)$ with a rational function of the form $\frac{c_1z^{\ell-1}+\dots + c_1}{z^{\ell-1}} = c_1z^{-(\ell-1)} + \dots + c_n $. The following is a sufficient condition on the characteristic polynomial $p(x)$ that guarantees the existence of such fitting, 

\begin{definition}
		A polynomial $p(z)$ of degree $n$ is extremely-acquiescent by extension of degree $d$ with bias $\epsilon$  if there exists a polynomial $u(z)$ of degree $d$ and leading coefficient 1 such that for all $z$ on unit circle, 
\begin{align}
\left|p(z)u(z)/z^{n+d}  -1\right|\le \epsilon\label{eqn:65}
\end{align}

\end{definition}

We remark that if $p(z)$ is $1$-acquiescent by extension of degree $d$, then there exists $u(z)$ such that $p(z)u(z)/z^{n+d} \in \ballC$. Therefore,  equation~\eqref{eqn:65} above is a much stronger requirement than acquiescence by extension.\footnote{We need $(1-\delta)$-acquiescence by extension in previous subsections for small $\delta > 0$, though this is merely additional technicality needed for the sample complexity. We ignore this difference between $1-\delta$-acquiescence and $1$-acquiescence and for the purpose of this subsection}

When $p(z)$ is extremely-acquiescent, we see that the transfer function $G(z) = s(z)/p(z)$ can be approximated by $s(z)u(z)/z^{n+d}$ up to bias $\epsilon$. Let $\ell = n+d+1$ and $s(z)u(z) = c_1z^{\ell-1} + \dots + c_{\ell}$. Then we have that $G(z)$ can be approximated by the following dynamical system of $\ell$ hidden states with $\epsilon$ bias: we choose $A = \CC(a)$ with $a_{\ell}=1$ and $a_{\ell-1}=\dots = a_1 = 0$, and $C = [c_1,\dots, c_{\ell}]$. As we have argued previously, such a dynamical system simply memorizes all the previous $\ell$ inputs, and therefore it is equivalent to linear regression from the feature  $[x_{k-\ell},\dots, x_{k}]$ to output $y_k$. 

\begin{proposition}[Informal]
	If the true system $G(z) = s(z)/p(z)$ satisfies that $p(z)$ is extremely-acquiescent by extension of degree $d$. Then using linear regression we can learn mapping from $[x_{k-\ell},\dots, x_{k}]$ to $y_{k}$ with bias $\epsilon$ and polynomial sampling complexity. 
\end{proposition}

We remark that with linear regression the bias $\epsilon$ will only go to zero as we increase the length $\ell$ of the feature, but not as we increase the number of samples. Moreover, linear regression requires a stronger assumption than the improper learning results in previous subsections do. 
The latter can be viewed as an interpolation between the proper case and the regime where linear regression works.

\section{Learning multi-input multi-output (MIMO) systems}
\label{sec:mimo}

\newcommand{\lin}{\ell_{\textup{in}}}
\newcommand{\lout}{\ell_{\textup{out}}}
We consider multi-input multi-output systems with the transfer functions that have a common denominator $p(z)$, 
\begin{equation}
G(z) = \frac{1}{p(z)} \cdot S(z) \label{eqn:form_mimo}
\end{equation}
where $S(z)$ is an $\lin\times \lout$ matrix with each entry being a polynomial with real coefficients of degree at most $n$ and $p(z) = z^n+a_1z^{n-1}+\dots + a_n$. 
Note that here we use $\lin$ to denote the dimension of the inputs of the system and $\lout$ the dimension of the outputs. 

Although a special case of a general MIMO system, this class of systems still contains many interesting cases, such as the transfer functions studied in~\cite{fazel2001rank,fazel2004rank}, where $G(z)$ is assumed to take the form 
$G(z) = R_0 + \sum_{i=1}^n \frac{R_i}{z-\lambda_i}\nonumber\,,$
for  $\lambda_1,\dots, \lambda_n\in \C$ with conjugate symmetry and 
$R_i\in \C^{\lout\times \lin}$ satisfies that $R_i =\bar{R}_j$ whenever $\lambda_i = \bar{\lambda}_j.$ 

In order to learn the system $G(z)$, we parametrize $p(z)$ by  its coefficients $a_1,\dots, a_n$ and $S(z)$ by the coefficients of its entries. Note that each entry of $S(z)$ depends on $n+1$ real coefficients and therefore the collection of coefficients forms a  third order tensor of dimension $\lout\times \lin\times (n+1)$. It will be convenient to collect the leading coefficients of the entries of $S(z)$ into a matrix of dimension $\lout\times \lin$, named $D$, and the rest of the coefficients into a matrix of dimension $\lout\times \lin n$, denoted by $C$.  This will be particularly intuitive when a state-space representation is used to learn the system with samples as discussed later. We parameterize the training transfer function $\hat{G}(z)$ by $\hat{a}$, $\hat{C}$ and $\hat{D}$ using the same way. 

Let's define the risk function in the frequency domain as,
\begin{equation}
g(\hatA,\hatC,\hat{D}) = \int_{0}^{2\pi} \left\|G(e^{i\theta})-\hat{G}(e^{i\theta})\right\|_F^2d\theta\label{eqn:mimo_risk}\mper
\end{equation}

The following lemma is an analog of Lemma~\ref{lem:main} for the MIMO case. Itss proof actually follows from a straightforward extension of the proof of Lemma~\ref{lem:main} by observing that matrix $S(z)$ (or $\hat{S}(z)$) commute with scalar $p(z)$ and $\hat{p}(z)$, and that $\hat{S}(z),\hat{p}(z)$ are linear in $\hata, \hatC$. 
\begin{lemma}\label{lem:mimo_lemma}
	The risk function $g(\hata,\hatC)$ defined in~\eqref{eqn:mimo_risk} is $\tau$-weakly-quasi-convex in the domain 
	\begin{equation}
		\neighbor(a) = \left\{
		\hata\in\mathbb{R}^n\colon 
		\Re\left(\frac{p_{a}(z)}{p_{\hata}(z)}\right) \ge \tau/2, \forall~z\in \C, \textup{~s.t.~} |z| = 1
		\right\} \otimes \R^{\lin\times \lout\times n'}\nonumber
	\end{equation}
\end{lemma}

Finally, as alluded before, we use a particular state space representation for learning the system in time domain with example sequences. It is known that any transfer function of the form~\eqref{eqn:form_mimo} can be realized uniquely by the state space system of the following special case of Brunovsky normal form~\cite{Brunovsky1970}, 

\begin{equation}
\quad A\; = \;
\left[ \begin{array}{ccccc} 0 & \Id_{\lin} & 0 & \cdots & 0 \\ 0 & 0 & \Id_{\lin} & \cdots & 0 \\
\vdots & \vdots & \vdots & \ddots & \vdots \\  0 & 0 & 0 & \cdots & \Id_{\lin} \\
-a_n\Id_{\lin} & -a_{n-1}\Id_{\lin} & -a_{n-2}\Id_{\lin} & \cdots & -a_1\Id_{\lin} \end{array} \right], \quad B = \left[ \begin{array}{c} 0 \\ \vdots \\ 0 \\ \Id_{\lin} \end{array} \right],  
\label{eqn:mimo_controllable_form}
\end{equation}
and, 
\begin{equation}
C\in \R^{\lout \times n\lin }, \quad D\in \R^{\lout\times\lin}\mper\nonumber
\end{equation}
The following Theorem is a straightforward extension of Corollary~\ref{cor:main} and Theorem~\ref{thm:improper} to the MIMO case.

\begin{theorem}\label{thm:mimo}
	Suppose transfer function $G(z)$ of a MIMO system takes form~\eqref{eqn:form_mimo}, and has norm $\|G\|_{\mathcal{H}_2}\le 1$. If the common denominator $p(z)$
			is \acext{\alpha}{d} then projected stochastic gradient descent over the state space representation~\eqref{eqn:mimo_controllable_form} will return $\hatTheta$ with risk
		\begin{equation}
			f(\hatTheta) \le \frac{\poly(n+d,\sigma,\tau, (1-\alpha)^{-1})}{TN}\mper\nonumber
		\end{equation}
\end{theorem}

We note that since $A$ and $B$ are simply the tensor product of $\Id_{\lin}$ with $\CC(a)$ and $e_n$,  the no blow-up property (Lemma~\ref{lem:no_blowing_up}) for $A^kB$ still remains true. Therefore to prove Theorem~\ref{thm:mimo}, we essentially only need to run the proof of Lemma~\ref{lem:variance} with matrix notation and matrix norm. We defer the proof to the full version.

\section{Simulations}

In this section, we provide proof-of-concepts experiments on synthetic data. We will demonstrate that 
\begin{enumerate}
	\item[1)] plain SGD tends to blow up even with relatively small learning rate, especially on hard instances
	\item[2)] SGD with our projection step converges with reasonably large learning rate, and with over-parameterization the final error is competitive
	\item[3)] SGD with gradient clipping has the strongest performance in terms both of the convergence speed and the final error
\end{enumerate} 

Here gradient clipping refers to the technique of using a normalized gradient instead of the true gradient. Specifically, for some positive hyper parameter $B$, we follow the approximate gradient
\[
g_{\textup{clip}} = \begin{cases} g & \mbox{if $\|g\|\le B$}\\
B g/\|g\| & otherwise
\end{cases}\,
\]
This method is commonly applied in training recurrent neural networks~\cite{pascanu2013difficulty}. 

Bullet 1) suggests that stability is indeed a real concern. Bullet 2) corroborates our theoretical study. Finding 3) suggests the instability of SGD partly arises from the noise in the batches, and such noise is reduced by the gradient clipping. Our experiments suggest that the landscape of the objective function may be even nicer than what is predicted by our theoretical development. It remains possible that the objective has no non-global local minima, possibly even outside the convex set to which our algorithm projects.

We generate the true system with state dimension $d=20$ by randomly picking the conjugate pairs of roots of the characteristic polynomial inside the circle with radius $\rho = 0.95$ and randomly generating the vector $C$ from standard normal distribution. 
The distribution of the norm of the impulse response $r$ (defined in Section~\ref{sec:fourier}) of such systems has a  heavy-tail. When the norm of $r$ is several magnitudes larger than the median it's  difficult to learn the system. Thus we select systems with reasonable $\|r\|$ for experiments, and we observe that the difficulty of learning increases as $\|r\|$ increases. 
The inputs of the dynamical model are generated from standard normal distribution with length $T=500$. We note that we generate new fresh inputs and outputs at every iterations and therefore the training loss is equal to the test loss (in expectation.) We use initial learning rate 0.01 in the projected gradient descent and SGD with gradient clipping. We use batch size 100 for all experiments, and decay the learning rate at 200K and 250K iteration by a factor of 10 in all experiments. 
\begin{figure}[t]
\centering

\begin{subfigure}[t]{0.32\textwidth}
	\centering
	\includegraphics[width=\textwidth]{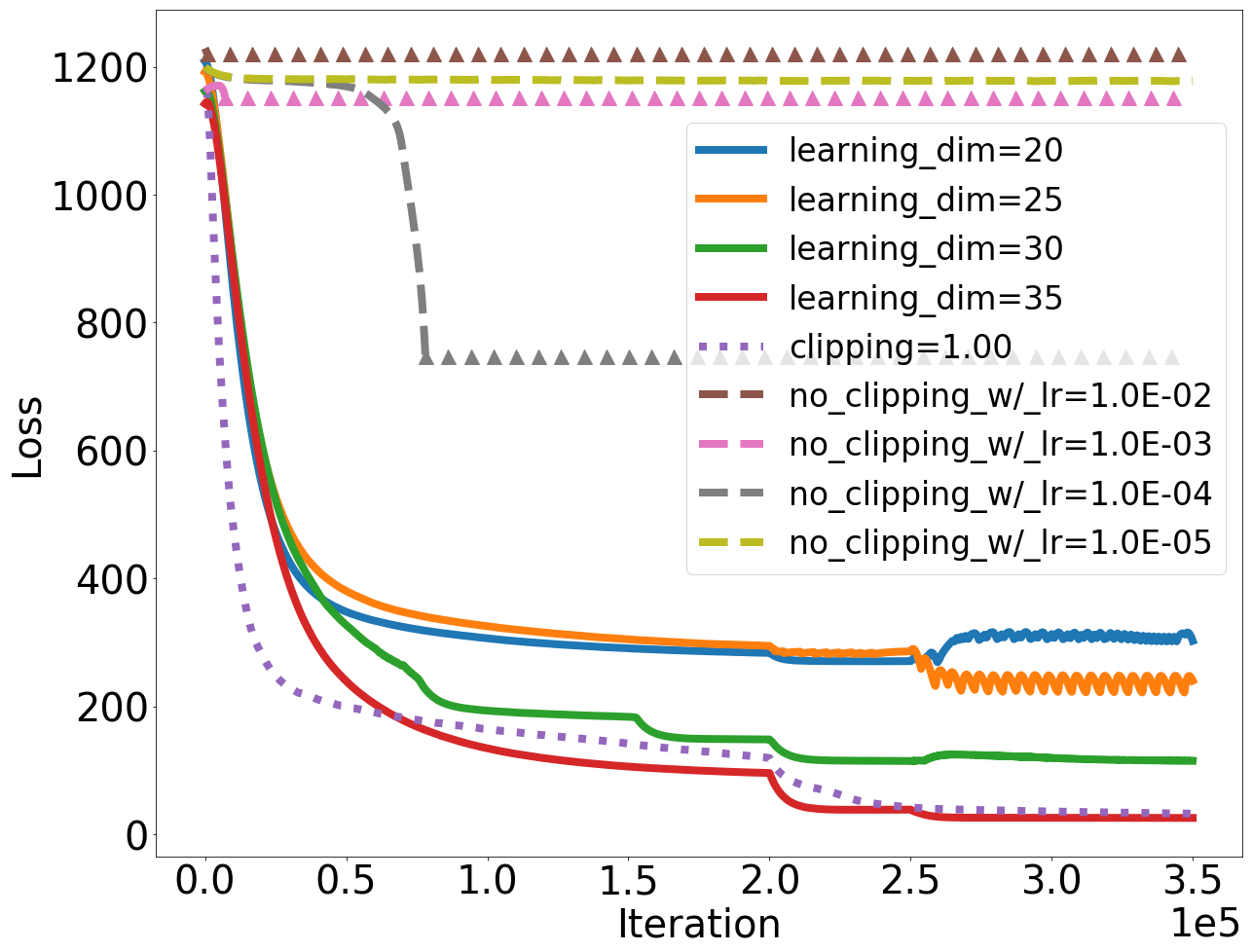}
	%\caption{An image from the GTA game}
\end{subfigure}
\begin{subfigure}[t]{0.32\textwidth}
	\centering
	\includegraphics[width=\textwidth]{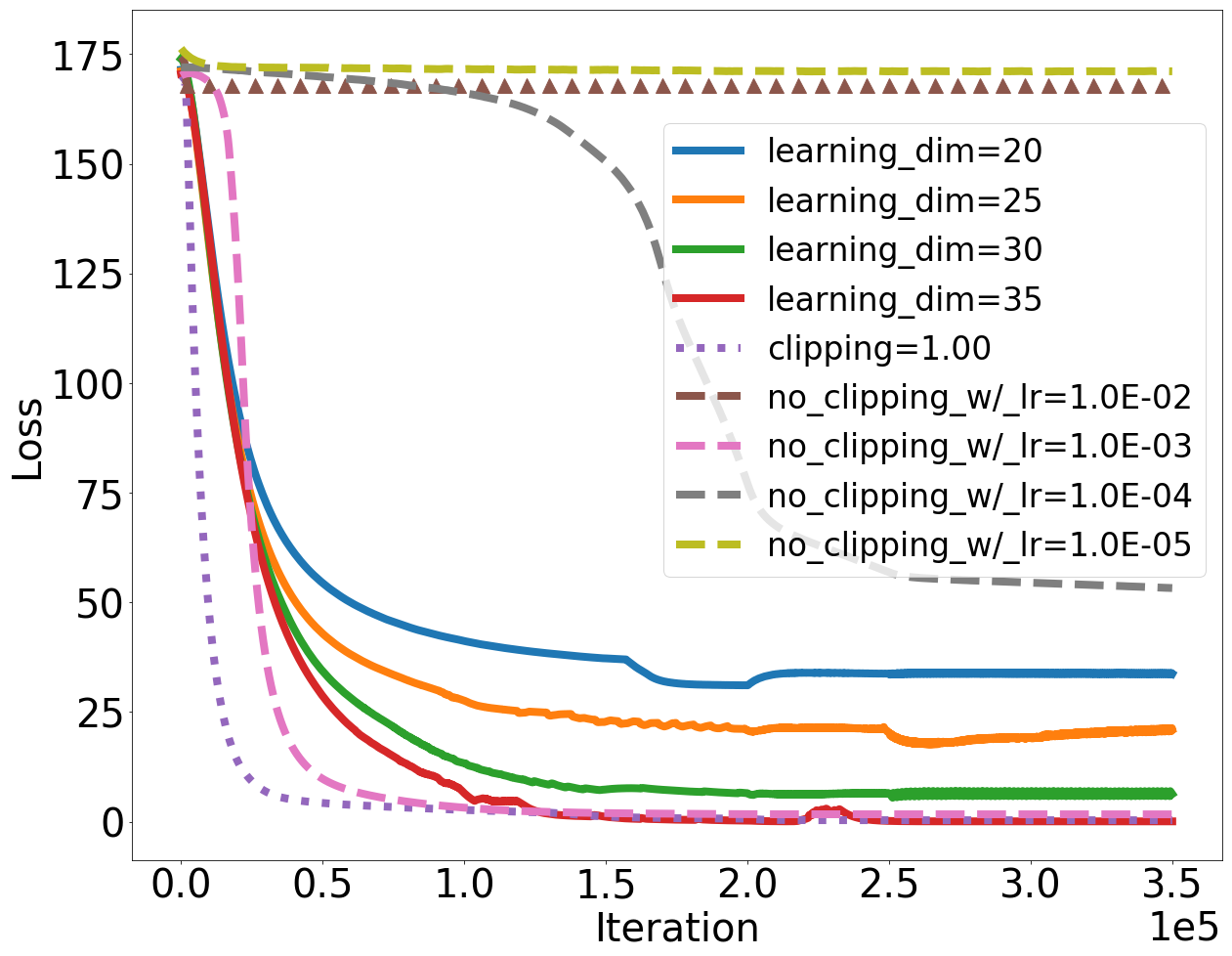}
	%\caption{An image from the GTA game}
\end{subfigure}
\begin{subfigure}[t]{0.32\textwidth}
	\centering
	\includegraphics[width=\textwidth]{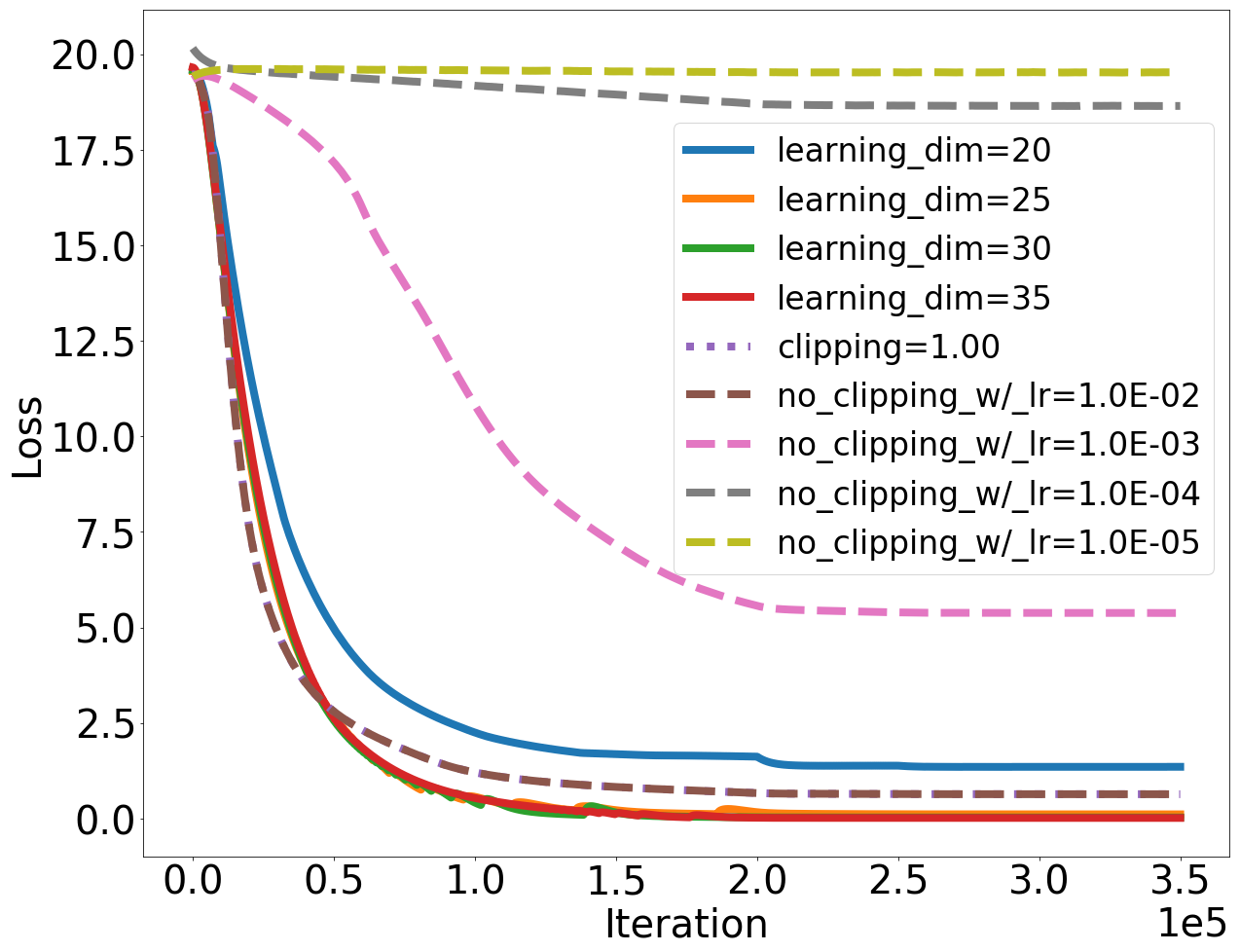}
	%\caption{An image from the GTA game}
\end{subfigure}
\caption{The performance of projected stochastic gradient descent with over-parameterization, vanilla SGD, and SGD with gradient clipping, on three different instance of dynamical systems with true state dimension = 20. The solid lines are from our proposed projected SGD with (over-parameterized) state dimension = 20, 25, 30, 35. The dot line corresponds to SGD with gradient clipped to Frobenius norm 1. The dashed lines correspond vanilla SGD and the triangle marker means the error blows up to infinity. The plot demonstrates the effect of the over-parameterization to our our algorithm. We note that the loss are different scales because the true systems in these three instances have different norms of impulse responses (which is equal to the loss of zero fitting). }
\label{fig:over}
\vskip -.5cm
\end{figure}

\section*{Acknowledgments}
We thank Amir Globerson, Alexandre Megretski, Pablo A. Parrilo, Yoram Singer, Peter Stoica, and Ruixiang Zhang for helpful discussions. We are indebted to Mark Tobenkin for pointers to relevant prior work.  We also thank Alexandre Megretski for helpful feedback, insights into passive systems and suggestions on how to organize Section~\ref{sec:fourier}.
\bibliography{ref}
\appendix
\section{Background on optimization}
\label{sec:optimization}

The proof below uses the standard analysis of gradient descent for non-smooth
objectives and demonstrates that the argument still works for \CONDNAME
functions.

\begin{proof}[Proof of Proposition~\ref{lem:framework}] 
	We start by using the
	\CONDNAME condition and then the rest follows a variant of the standard analysis of non-smooth
	projected sub-gradient descent\footnote{Although we used weak smoothness to get a slightly better bound}. We conditioned on $\thetax_k$, and have that
	\begin{align}
	\tau (f(\thetax_k) - f(\thetax^*)) & \le \nabla f(\thetax_k)^{\top} (\thetax_k-\thetax^*) = \Exp[\mfr(\thetax_k)^{\top}(\thetax_k-\thetax^*)\mid \thetax_k ]\nonumber \\
	& = \Exp\left[\frac{1}{\eta} (\thetax_k -w_{k+1}) (\thetax_k-\thetax^*)\mid \thetax_k\right] \nonumber\\
	& = \frac{1}{\eta} \left(\Exp\left[\|\thetax_k-w_{k+1}\|^2 \mid \thetax_k\right]+ \|\thetax_k - \thetax^*\|^2 - \Exp\left[\|w_{k+1}-\thetax^*\|^2\mid \thetax_k\right]\right)\nonumber \\
	& = \eta \Exp\left[\|\mfr(\thetax_k)\|^2 \right]+ \frac{1}{\eta}\left(\|\thetax_k - \thetax^*\|^2 - \Exp\left[\|w_{k+1}-\thetax^*\|^2\mid \thetax_k\right]\right) \label{eqn:eqn3}
	\end{align}
	where the first inequality uses \CONDNAME and the rest of lines are simply algebraic manipulations. Since $\thetax_{k+1}$ is the projection of $w_{k+1}$ to $\ball$ and $\thetax^*$ belongs to $\ball$, we have $\|w_{k+1}-\thetax^*\|\ge \|\thetax_{k+1}-\thetax^*\|$. Together with~\eqref{eqn:eqn3}, and $$\Exp\left[\|\mfr(\thetax_k)\|^2\right] = \|\nabla f(\thetax_k)\|^2 + \Var[\mfr(\thetax_k)] \le \Gamma (f(\thetax_k)-f(\thetax^*)) + V, $$ we obtain that
	\begin{align*}
	\tau (f(\thetax_k) - f(\thetax^*))
	& \le \eta \Gamma(f(\thetax_k) - f(\thetax^*))+ \eta V + \frac{1}{\eta}\left(\|\thetax_k - \thetax^*\|^2 - \Exp\left[\|\thetax_{k+1}-\thetax^*\|^2\mid \thetax_k\right]\right)\mper
	\end{align*}
	Taking expectation over all the randomness and summing over $k$ we obtain that
	\begin{align*}
	\sum_{k=0}^{K-1} \Exp\left[f(\thetax_k) - f(\thetax^*)\right] &\le \frac{1}{\tau-\eta\Gamma}\left(\eta KV + \frac{1}{\eta}\|\thetax_0-\thetax^*\|^2\right)
	\le \frac{1}{\tau-\eta\Gamma}\left(\eta KV + \frac{1}{\eta}R^2\right)\mper
	\end{align*}
	where we use the assumption that $\|\thetax_0-\thetax^*\|\le R$. Suppose $K \ge \frac{4R^2\Gamma^2}{ V\tau^2}$, then we take $\eta = \frac{R}{\sqrt{VK}}$. Therefore we have that $\tau - \eta \Gamma \ge \tau/2$ and therefore 
		\begin{equation}
	\sum_{k=0}^{K-1} \Exp\left[f(\thetax_k) - f(\thetax^*)\right] \le \frac{4R\sqrt{V}\sqrt{K}}{\tau} \mper\label{eqn:eqn9}
	\end{equation}
	
	On the other hand, if $K \le \frac{4R^2\Gamma^2}{ V\tau^2}$, we pick $\eta = \frac{\tau}{2\Gamma}$ and obtain that 
	\begin{equation}
	\sum_{k=0}^{K-1} \Exp\left[f(\thetax_k) - f(\thetax^*)\right] \le \frac{2}{\tau}\left(\frac{\tau KV}{2\Gamma} + \frac{2\Gamma R^2}{\tau}\right) \le \frac{8\Gamma R^2}{\tau^2}\mper\label{eqn:eqn8}
	\end{equation}
	
	Therefore using equation~\eqref{eqn:eqn8} and~\eqref{eqn:eqn9} we obtain that when choosing $\eta$ properly according to $K$ as above, 
	\begin{equation*}
	\Exp_{k\in [K]}\left[f(\thetax_k) - f(\thetax^*)\right] \le \max\left\{\frac{8\Gamma R^2}{\tau^2K}, \frac{4R\sqrt{V}}{\tau \sqrt{K}}\right\}\mper\qedhere
	\end{equation*}
\end{proof}

\section{Toolbox}

\begin{lemma}\label{lem:uinverseB}
		
		Let $B = e_n\in \R^{n\times 1}$ and $\lambda \in [0,2\pi]$, $w\in \C$. Suppose $A$ with $\rho(A)\cdot|w| < 1$ has the controllable canonical form $A  =\CC(a)$. Then  
		$$(I-wA)^{-1}B =\frac{1}{p_a(w^{-1})}\left[\begin{matrix} w^{-1}\\w^{-2}\\\vdots \\ w^{-n}\end{matrix}\right]$$
		where $p_a(x) = x^n + a_1x^{n-1}+\dots+a_n$ is the characteristic polynomial of $A$. 
	\end{lemma}
		\begin{proof}
		let $v =  (I-wA)^{-1}B$ then we have $(I-wA)v = B$. Note that $B= e_n$, and $I-wA$ is of the form
			\begin{equation}
			I-wA \; = \;
			\left[ \begin{array}{ccccc} 1 & -w & 0 & \cdots & 0 \\ 0 & 1 & -w & \cdots & 0 \\
			\vdots & \vdots & \vdots & \ddots & \vdots \\  0 & 0 & 0 & \cdots & -w \\
			a_nw & a_{n-1}w & a_{n-2}w & \cdots & 1+ a_1 w\end{array} \right]
			\end{equation}
		Therefore we obtain that $v_{k} = wv_{k+1}$ for $1\le k \le n-1$. That is, $v_k = v_0 w^{-k}$ for $v_0 = v_1 w^{1}$. Using the fact that $((I-wA)v)_n = 1$, we obtain that $v_0 = p_a(w^{-1})^{-1}$ where $p_a(\cdot)$ is the polynomial $p_a(x) = x^n+ a_1x^{n-1} + \dots + a_n$. Then we have that $u(I-wA)^{-1}B =\frac{u_1w^{-1} + \dots + u_n w^{-n}}{p_a(w^{-1})}$
	\end{proof}
	
\begin{lemma}\label{claim:gaussian_expectation}
	Suppose $x_1,\dots,x_n$ are independent variables with mean 0 and covariance matrices and  $\Id$, $U_1,\dots, U_d$ are fixed matrices, then
	\begin{equation}\textstyle
	\Exp\left[\|\sum_{k=1}^{n} U_kx_k\|^2\right]  = \sum_{k=1}^n \|U_k\|_F^2\mper\nonumber
	\end{equation}
\end{lemma}
\begin{proof}
	We have that 
	\begin{equation*}\textstyle
	\Exp\left[\|\sum_{k=1}^{n} U_kx_k\|_F^2\right]  = \Exp \sum_{k, \ell}^n \trace(U_k x_kx_{\ell}^{\top}U_{\ell}^{\top}) = \sum_{k}^n \trace(U_k x_kx_{k}^{\top}U_{k}^{\top}) = \sum_{k=1}^n \|U_k\|_F^2\qedhere
	\end{equation*}
\end{proof}

\section{Missing proofs in Sections~\ref{sec:acq} and~\ref{sec:learning}}
\label{sec:proofs}

\subsection{Monotonicity of acquiescence: Proof of Lemma~\ref{lem:monotone}}
\label{sec:monotonicity}

\begin{lemma}[Lemma~\ref{lem:monotone} restated]\label{lem:monotone_repeat}
	For any $0 < \alpha < \beta$, we have that $\ball_{\alpha}\subset \ball_{\beta}$. 
\end{lemma}

\begin{proof} Let $q_a(z) = 1+a_1z+\dots + a_nz^n$. Note that $q(z^{-1}) = p_a(z)/z^n$. Therefore we note that $B_{\alpha} = \{a: q_a(z)\in \ballC, \forall |z| = 1/\alpha \}$. 
	Suppose $a\in \ball_{\alpha}$, then $\Re(\phat_{a}(z)) \ge \tau_1$  for any $z$ with $|z| = 1/\alpha$. 
	Since  $\Re(\phat_a(z))$ is the real part of the holomorphic function $\phat_a(z)$, its a harmonic function. By maximum (minimum) principle of the harmonic functions, we have that for any $|z| \le 1/\alpha$, $\Re(\phat_a(z)) \ge \inf_{|z| = 1/\alpha} \Re(\phat_a(z)) \ge \tau_1$. In particular, it holds that for $|z| = 1/\beta <1/\alpha$, $\Re(\phat_a(z))\ge \tau_1$. Similarly we can prove that for $z$ with $|z| = 1/\beta$, $\Re(q_a(z)) \ge (1+\tau_0)\Im(q_a(z))$, and other conditions for $a$ being in $\ball_{\beta}$. 
\end{proof}

\subsection{Proof of Lemma~\ref{lem:unbiased_estimator}}

Lemma~\ref{lem:unbiased_estimator} follows directly from the following general Lemma  which also handles the multi-input multi-output case. It can be seen simply from calculation similar to the proof of Lemma~\ref{lem:population_risk}. We mainly need to control the tail of the series using the no-blow up property (Lemma~\ref{lem:no_blowing_up})  and argue that the wrong value of the initial states $h_0$ won't cause any trouble to the partial loss function $\ell((x,y),\hatTheta)$ (defined in Algorithm~\ref{alg:backprop}). This is simply because  after time $T_1 = T/4$, the influence of the initial state is already washed out.

\begin{lemma}\label{lem:unbisaed_estimator_general}
	In algorithm~\ref{alg:back_prop_general} the values of $G_A,G_C,G_D$ are equal to the gradients of $g(\hatA,\hatC)+(\hatD-D)^2$ 	with respect to $\hatA$, $\hatC$ and $\hatD$ up to inverse exponentially small error. 
\end{lemma}
\begin{proof}[Proof of Lemma~\ref{lem:unbisaed_estimator_general}]
			We first show that the partial empirical loss function $\ell((x,y),\hatTheta)$ has expectation almost equal to the idealized risk (up to the term for $\hatD$ and exponential small error), $$\Exp[\ell((x,y),\hatTheta)] = g(\hatA,\hatC) + (\hatD-D)^2\pm \exp(-\Omega((1-\alpha)T)).$$ 
	This can be seen simply from similar calculation to the proof of Lemma~\ref{lem:population_risk}. 	Note that 
	\begin{align}
	y_t = Dx_t + \sum_{k=1}^{t-1} CA^{t-k-1}Bx_{k}  + CA^{t-1}h_0+ \xi_t
	\quad\text{and}\quad
	\tilde{y}_t = \hatD x_t + \sum_{k=1}^{t-1} \hatC \hatA^{t-k-1}\hatB x_{k} \label{eqn:tildey_t}\,.
	\end{align}
	Therefore noting that when $t\ge T_1\ge \Omega(T)$, we have that $\|CA^{t-1}h_0\| \le \exp(-\Omega((1-\alpha)T)$ and therefore the effect of $h_0$ is negligible. Then we have that 
	\begin{align}
	\Exp[\ell((x,y),\hatTheta)]  & = 
	\frac{1}{T-T_1}\Exp\left[\sum_{t >  T_1}^{T} \|y_t-y_t\|^2\right]\pm \exp(-\Omega((1-\alpha)T)) \nonumber\\
	& = \|\hatD-D\|^2 + \frac{1}{T-T_1}\sum_{T\ge t > T_1} \sum_{0\le j\le t-1} \|\hatC\hatA^{j}B - CA^{j}B\|^2 \pm \exp(-\Omega((1-\alpha)T))\nonumber\\
		& = \|\hatD-D\|^2 + \sum_{j=0}^{T_1}\|\hatC\hatA^{j}B - CA^{j}B\|^2  + \sum_{T\ge j \ge T_1} \frac{T-j}{T-T_1}\|\hatC\hatA^{j}B - CA^{j}B\|^2 \pm \exp(-\Omega((1-\alpha)T)) \nonumber\\
& = \|\hatD-D\|^2 + \sum_{j=0}^{\infty}\|\hatC\hatA^{j}B - CA^{j}B\|^2 \pm \exp(-\Omega((1-\alpha)T)) \nonumber\mper
	\end{align}
	where the first line use the fact that $\|CA^{t-1}h_0\| \le \exp(-\Omega((1-\alpha)T)$, the second uses equation~\eqref{eqn:37} and the last line uses the no-blowing up property of $A^kB$ (Lemma~\ref{lem:no_blowing_up}). 
	
	Similarly, we can prove that the gradient of $\Exp[\ell((x,y),\hatTheta)]$ is also close to the gradient of $g(\hatA,\hatC) + (\hatD-D)^2$ up to inverse exponential error. 
\end{proof}

\subsection{Proof of Lemma~\ref{lem:variance}}
\newcommand{\undecided}{n}

\begin{proof}[Proof of Lemma~\ref{lem:variance}] 	Both $G_A$ and $G_C$ can be written in the form of a quadratic form (with vector coefficients) of $x_1,\dots,x_T$ and $\xi_1,\dots,\xi_T$. That is, we will write 
	\begin{equation*}
	G_A  = \sum_{s,t}  x_sx_t u_{st}+ \sum_{s, t} x_s\xi_t u_{st}'\quad\text{and}\quad
	G_C  = \sum_{s,t} x_sx_tv_{st}  + \sum_{s, t} x_s\xi_t v_{st}' \,.
	\end{equation*}
	
	where $u_{st}$ and $v_{st}$ are vectors that will be calculated later. By Lemma~\ref{claim:quadratic_form_variance}, we have that 
	\begin{equation}
	\Var\left[\sum_{s,t} x_sx_su_{st} + \sum_{s,t} x_s\xi_tu'_{st}\right] \le O(1)\sum_{s,t} \|u_{st}\|^2 + O(\sigma^2)\sum_{s,t} \|u'_{st}\|^2\mper \label{eqn:variance_formula}
	\end{equation}

	Therefore in order to bound from above $\Var\left[G_A\right]$, it suffices to bound $\sum \|u_{st}\|^2$ and $\sum \|u'_{st}\|^2$, and similarly for $G_C$. 
	
	We begin by writing out $u_{st}$ for fixed $s,t\in [T]$ and bounding its norm. We use the same set of notations as int the proof of Lemma~\ref{lem:unbiased_estimator}.  Recall that we set $\rq_k = CA^kB$ and $\hat\rq_k = \hatC\hatA^k B$, and $\drq_k = \hat\rq_k - \rq_k$. Moreover, let $z_k = \hatA^k B$.  									We note that the sums of $\|z_k\|^2$ and $r_k^2$ can be controlled. 	By the assumption of the Lemma, we have that  
	\begin{align}
	\sum_{k=t}^{\infty} \|z_k\|^2  	& \le 2\pi n\tau_1^{-2}\label{eqn:zk_tail_bound_2}, ~\textrm{}\quad \|z_k\|^2 \le  2\pi n \alpha^{2k-2\undecided}\tau_1^{-2}\mper\\
	\sum_{k=t}^{\infty} \dr_k^2 	&\le 4\pi n\tau_1^{-2}, \quad \|\dr_k\|^2 \le  4\pi n \alpha^{2k-2\undecided}\tau_1^{-2}\mper\label{eqn:delta_rk_tail_bound_2}
	\end{align}
	which will be used many times in the proof that follows. 
	
	We calculate the explicit form of $G_A$ using the explicit back-propagation Algorithm~\ref{alg:back_prop_general}. We have that in Algorithm~\ref{alg:back_prop_general}, 
	\begin{equation}\textstyle
	\hath_k = \sum_{j=1}^{k} \hatA^{k-j}Bx_j =  \sum_{j=1}^{k} z_{k-j}x_j\label{eqn:hkprime}
	\end{equation}
	and 
		\begin{align}
		\dh_k &= \sum_{j=k}^{T}\tp{\hatA}{j-k}\hatC^{\top}\dy_j 
		= \sum_{j=k}^{T}\alpha^{j-k}\tp{\hatA}{j-k}\hatC^{\top}\indicator{j > T_1}\left(\xi_j + \sum_{\ell = 1}^{j}  \drq_{j-\ell}x_\ell\right)\label{eqn:dh_k}
	\end{align}
	Then using $G_A =  \sum_{k\ge 2} B^{\top}\dh_k h_{k-1}^{'\top}$ and equation~\eqref{eqn:hkprime} and equation~\eqref{eqn:dh_k} above, we have that 
	\begin{align}
	u_{st} & = \textstyle\sum_{k = 2}^{T} \left(\sum_{j \ge \max\{k,s,T_1+1\}} \drq_{j-s}\hatC \hatA^{j-k}B\right)\indicator{k\ge t+1}\cdot\hatA^{k-t-1}B \nonumber \\
	& = \textstyle\sum_{k = 2}^{T} \left(\sum_{j \ge \max\{k,s,T_1+1\}} \drq_{j-s}\hat\rq_{j-k}\right)\indicator{k\ge t+1}\cdot z_{k-t-1}\mper \label{eqn:ust}
	\end{align}
	and that, 	
	\begin{align}
	u_{st}' & = \sum_{k=2}^T z_{k-1-s}\cdot \indicator{k \ge s+1 } \cdot \hat{\rq}'_{t-k}\cdot \indicator{t > \max\{T_1, k\}}
	 = \sum_{s+1 \le k \le t} z_{k-1-s}\cdot \hat{\rq}'_{t-k}\cdot \indicator{t > \max\{T_1\}}\label{eqn:ustprime}
	\end{align}
	
	Towards bounding $\|u_{st}\|$, we consider four different cases.  Let $\Lambda = \Omega\left(\{\max\{\undecided,(1-\alpha)^{-1}\log(\frac{1}{1-\alpha})\}\right)$ be a threshold. 
	
\paragraph{Case 1:} When $0 \le s -t \le \Lambda$, we rewrite $u_{st}$ by rearranging equation~\eqref{eqn:ust}, 
		\begin{align}
		u_{st}& = \sum_{T\ge k\ge s} z_{k-t-1}\sum_{j \ge \max\{k,T_1+1\}} \drq_{j-s}\hat\rq_{j-k}  + \sum_{t < k < s} z_{k-t-1} \sum_{j \ge \max\{s,T_1+1\}} \drq_{j-s}\hat\rq_{j-k} \nonumber\\
		&= \sum_{\ell \ge 0, \ell \ge T_1+1-s} \drq_{\ell} \sum_{s\le k\le l+s, k\le T} \hat\rq_{\ell+s-k}z_{k-t-1} + \sum_{\ell \ge 0, \ell \ge T_1+1-s} \drq_{\ell}\sum_{s > k > t}\hat\rq_{\ell+s-k}z_{k-t-1}\nonumber
		\end{align}
		where at the second line, we did the change of variables $\ell = j-s$. 		Then by Cauchy-Schartz inequality, we have,
		\begin{align}
		\|u_{st}\|^2  &\le  2\left(\sum_{\ell \ge 0, \ell \ge T_1+1-s} \drq_{\ell}^2\right) \underbrace{\left(\sum_{\ell \ge 0, \ell \ge T_1+1-s}\left\|\sum_{s\le k\le l+s, k\le T} \hat\rq_{\ell+s-k}z_{k-t-1}
			\right\|^2\right)}_{T_1} \nonumber\\
		& \quad + 2\left(\sum_{\ell \ge 0, \ell \ge T_1+1-s} \drq_{\ell}^2\right)\underbrace{\left(\sum_{\ell \ge 0, \ell \ge T_1+1-s}\left\|\sum_{s > k > t}\hat\rq_{\ell+s-k}z_{k-t-1}\right\|^2\right)}_{T_2}\mper\label{eqn:ust2}
		\end{align}
		
		We could bound the contribution from $\drq_k^2$ ssing equation~\eqref{eqn:delta_rk_tail_bound_2}, and it								remains to bound terms $T_1$ and $T_2$. Using the tail bounds for $\|z_k\|$
		(equation~\eqref{eqn:zk_tail_bound_2}) and the fact that $|\hat{\rq}_k| = |\hatC\hatA^k B|\le \|\hatA^kB\| = \|z_k\|$ , we have that 
		\begin{align}
		T_1 = \sum_{\ell \ge 0, \ell \ge T_1+1-s}\left\|\sum_{s\le k\le l+s, k\le T} \hat\rq_{\ell+s-k}z_{k-t-1}\right\|^2 &\le  \sum_{\ell \ge 0} \left(\sum_{s\le k\le \ell+s} |\hat\rq_{\ell+s-k}|\|z_{k-t-1}\|\right)^2 \mper\label{eqn:eqn14}
			\end{align}

		We bound the inner sum of RHS of~\eqref{eqn:eqn14} using the fact that $\|z_k\|^2\le O(n\alpha^{2k-2\undecided}/\tau_1^2)$ and obtain that, 
		\begin{align}
		\sum_{s\le k\le \ell+s} |\hat\rq_{\ell+s-k}|\|z_{k-t-1}\|& \le \sum_{s\le k\le \ell+s}O(n\alpha^{(\ell+s-t-1)-2\undecided}/\tau_1^2) \nonumber\\
		& \le O(\ell n\alpha^{(\ell+s-t-1)-2\undecided}/\tau_1^2)\mper \label{eqn:eqn17}
		\end{align}
		Note that equation~\eqref{eqn:eqn17} is particular effective when $\ell > \Lambda$. 		When $\ell \le \Lambda$, we can refine the bound using equation~\eqref{eqn:zk_tail_bound_2} and obtain that 
		\begin{align}
		\sum_{s\le k\le \ell+s} |\hat\rq_{\ell+s-k}|\|z_{k-t-1}\|& \le \left(\sum_{s\le k\le \ell+s}|\hat{\rq}_{\ell+s-k}|^2 \right)^{1/2}
		\left(\sum_{s\le k\le \ell+s} \|z_{k-t-1}\|^2\right)^{1/2} \nonumber\\
		& \le O(\sqrt{n}/\tau_1)\cdot O(\sqrt{n}/\tau_1) = O(n/\tau_1^2)\mper\label{eqn:eqn16}
		\end{align}
				
		Plugging equation~\eqref{eqn:eqn16} and~\eqref{eqn:eqn17} into equation~\eqref{eqn:eqn14}, we have that 
		\begin{align}
		\sum_{\ell \ge 0} \left(\sum_{s\le k\le \ell+s} |\hat{\rq}_{\ell+s-k}|\|z_{k-t-1}\|\right)^2  & \le \sum_{\Lambda\ge \ell\ge 0 } O(n^2/\tau_1^4) + \sum_{\ell > \Lambda}O(\ell^2 n^2\alpha^{2(\ell+s-t-1)-4\undecided}/\tau_1^4) \nonumber \\
		& \le O(n^2\Lambda/\tau_1^4) + O(n^2/\tau_1^4) = O(n^2\Lambda/\tau_1^4)\mper \label{eqn:firstpart}
		\end{align}
				For the second term in equation~\eqref{eqn:ust2}, we bound similarly, 
		\begin{align}
		T_2 \le \sum_{\ell \ge 0, \ell \ge T_1+1-s}\left\|\sum_{s > k > t}\hat\rq_{\ell+s-k}z_{k-t-1}\right\|^2 & \le O(n^2\Lambda/\tau_1^4)\mper\label{eqn:secondpart}
		\end{align}
		Therefore using the bounds for $T_1$ and $T_2$ we obtain that, 
		\begin{align}
		\|u_{st}\|^2 \le O(n^3\Lambda/\tau_1^6)\label{eqn:ust_case1}
		\end{align}
		
\paragraph{Case 2:} 
When $s -t > \Lambda$, we tighten equation~\eqref{eqn:firstpart} by observing that, 
		\begin{align}
		T_1 \le \sum_{\ell \ge 0} \left(\sum_{s\le k\le \ell+s} |\hat{\rq}_{\ell+s-k}|\|z_{k-t-1}\|\right)^2  &\le \alpha^{2(s-t-1)-4\undecided}\sum_{\ell\ge 0}O(\ell^2 n^2\alpha^{2\ell}/\tau_1^4)\nonumber\\
		& \le \alpha^{s-t-1}\cdot O(n^2/(\tau_1^4(1-\alpha)^3))  \mper
		\end{align}
		where we used equation~\eqref{eqn:eqn17}. Similarly we can prove that 
		\begin{align*}
		T_2 \le \alpha^{s-t-1}\cdot O(n^2/(\tau_1^4(1-\alpha)^3)) \mper 
		\end{align*}
		Therefore, we have when $s-t \ge \Lambda$, 
		\begin{align}
		\|u_{st}\|^2 \le O(n^3/((1-\alpha)^3 \tau_1^6)) \cdot \alpha^{s-t-1}\label{eqn:ust_case2}\mper
		\end{align}
		
\paragraph{Case 3:}
When $-\Lambda\le s-t\le 0$, we can rewrite $u_st$ and use the Cauchy-Schwartz inequality and obtain that 
		\begin{align}
		u_{st} & = \sum_{T\ge k\ge t+1} z_{k-t-1}\sum_{j \ge \max\{k,T_1+1\}} \drq_{j-s}\hat\rq_{j-k}= \sum_{\ell \ge 0, \ell \ge T_1+1-s} \drq_{\ell} \sum_{t+1\le k\le l+s, k\le T} \hat\rq_{\ell+s-k}z_{k-t-1} \mper\nonumber
		\end{align}
		and, 
		\begin{align}
		\|u_{st}\|^2 & \le \left(\sum_{\ell \ge 0, \ell \ge T_1+1-s} \drq_{\ell}^2 \right)\left(\sum_{\ell \ge 0, \ell \ge T_1+1-s}\left\|\sum_{t+1\le k\le l+s, k\le T} \hat\rq_{\ell+s-k}z_{k-t-1}
		\right\|^2\right)\mper\nonumber
		\end{align} 
		Using almost the same arguments as in equation~\eqref{eqn:eqn17} and~\eqref{eqn:eqn16}, we that 
		\begin{align*}
		&\sum_{t+1\le k\le \ell+s} |\hat{\rq}_{\ell+s-k}|\cdot \|z_{k-t-1}\|
		 \le O(\ell n\alpha^{(\ell+s-t-1)-2\undecided}/\tau_1^2)\\	
		\text{and}\qquad&\sum_{t+1\le k\le \ell+s} |\hat{\rq}_{\ell+s-k}|\cdot\|z_{k-t-1}\|
		 \le O(\sqrt{n}/\tau_1)\cdot O(\sqrt{n}/\tau_1) = O(n/\tau_1^2)\mper
		\end{align*}
		Then using a same type of argument as equation~\eqref{eqn:firstpart}, we can have that 
		\begin{align}
		\sum_{\ell \ge 0, \ell \ge T_1+1-s}\left\|\sum_{t+1\le k\le l+s, k\le T} \hat\rq_{\ell+s-k}'z_{k-t-1}'
		\right\|^2  	& \le O(n^2\Lambda/\tau_1^4) + O(n^2/\tau_1^4) \nonumber\\
		&= O(n^2\Lambda/\tau_1^4)\mper \nonumber 		\end{align} 
		It follows that in this case $\|u_{st}\|$ can be bounded with the same bound in~\eqref{eqn:ust_case1}. 
				\paragraph{Case 4:} 
When $s-t \le -\Lambda$, we use a different simplification of $u_{st}$ from above. 
				First of all, 	it follows~\eqref{eqn:ust} that \begin{align}
		\left\|u_{st}\right\| &\le \sum_{k = 2}^{T} \left(\sum_{j \ge \max\{k,s,T_1+1\}} \|\drq_{j-s}\hat\rq_{j-k}'z_{k-t-1}\| \indicator{k\ge t+1}\right)\nonumber\\
				\\		& \le \sum_{k \ge t+1} \|z_{k-t-1}'\|\sum_{j \ge \max\{k,T_1+1\}} |\drq_{j-s}\hat\rq_{j-k}'| \mper\nonumber	\end{align}
		Since $j-s\ge k -s > 4n$ and it follows that 
		\begin{align}
		\sum_{j \ge \max\{k,T_1+1\}} |\drq_{j-s}\hat\rq_{j-k}'| &\le \sum_{j \ge \max\{k,T_1+1\}}  O(\sqrt{n}/\tau_1\cdot \alpha^{j-s-\undecided})\cdot O(\sqrt{n}/\tau_1\cdot \alpha^{j-k-\undecided})\nonumber\\
		& \le O(n/(\tau_1^2(1-\alpha))\cdot \alpha^{k-s-\undecided}) \nonumber		\end{align}	
		Then we have that 
		\begin{align}
		\left\|u_{st}\right\|^2 	& \le \sum_{k \ge t+1} \|z_{k-t-1}'\|\sum_{j \ge \max\{k,T_1+1\}} |\drq_{j-s}\hat\rq_{j-k}'| \nonumber\\
		& \le  \left(\sum_{k \ge t+1} \|z_{k-t-1}'\|^2\right)\left(\sum_{k \ge t+1} \left(\sum_{j \ge \max\{k,T_1+1\}} |\drq_{j-s}\hat\rq_{j-k}'| \right)^2\right)\nonumber\\
		& \le O(n/\tau_1^2) \cdot O(n^2/(\tau_1^4(1-\alpha)^3) \alpha^{t-s}) = O(n^3/(\tau_1^6\delta^3) \alpha^{t-s})\nonumber
		\end{align}
		
				Therefore, using the bound for $\|u_{st}\|^2$ obtained in the four cases above, taking sum over $s,t$, we obtain that 
		\begin{align}
		\sum_{1\le s,t\le T}\|u_{st}\|^2 &\le \sum_{s,t\in [T]: |s-t|\le \Lambda} O(n^3\Lambda/\tau_1^6) + \sum_{s,t: |s-t| \ge \Lambda}O(n^3/(\tau_1^6(1-\alpha)^3) \alpha^{|t-s|-1}) \nonumber\\
		&\le O(Tn^3\Lambda^2/\tau_1^6) + O(n^3/\tau_1^6) = O(Tn^3\Lambda^2/\tau_1^6) \mper\label{eqn:sum_ust}
		\end{align}

We finished the bounds for $\|u_{st}\|$ and now we turn to bound $\|u'_{st}\|^2$. Using the formula for $u_{st}'$ (equation~\ref{eqn:ustprime}), we have that 	for $t \le s+1$, $u_{st}' = 0$.  For $s+ \Lambda\ge t \ge s+2$, we have that by Cauchy-Schwartz inequality, 
	\begin{align*}
	\|u_{st}'\|  & \le \left(\sum_{s+1 \le k \le t} \|z_{k-1-s}\|^2\right)^{1/2} \left(\sum_{s+1 \le k \le t} |\hat{\rq}'_{t-k}|^2\right)^{1/2}\le O(n/\tau_1^2)\le O(n/\tau_1^2)\mper
	\end{align*}
	On the other hand, for $t > s+\Lambda$, by the bound that $|\hat{\rq}_k'|^2\le \|z_k'\|^2\le O(n\alpha^{2k-2\undecided}/\tau_1^2)$,  we have,
	\begin{align*}
	\|u_{st}'\|  & \le \sum_{s+1 \le k \le t-1}^T \|z_{k-1-s}\|\cdot |\hat{r}'_{t-k}| \le \sum_{s+1 \le k \le t-1}^T n\alpha^{t-s-1}/\tau_1^2 \\
	& \le O(n(t-s)\alpha^{t-s-1}/\tau_1^2) \mper	\end{align*}
	Therefore taking sum over $s,t$, similarly to equation~\eqref{eqn:sum_ust}, 
	\begin{align}
	\sum_{s,t\in [T]} \|u_{st}'\|^2 & \le O(Tn^2\Lambda/\tau_1^4)\mper\label{eqn:sum_ustprime}	\end{align}
	Then using equation~\eqref{eqn:variance_formula} and equation~\eqref{eqn:sum_ust} and~\eqref{eqn:sum_ustprime}, we obtain that 
	$$\Var[\|G_A\|^2] 	\le 
	O\left(Tn^3\Lambda^2/\tau_1^6+ \sigma^2Tn^2\Lambda/\tau_1^4\right). $$
	Hence, it follows that 
	$$\Var[G_A] \le \frac{1}{(T-T_1)^2}\Var[G_A] \le \frac{O\left(n^3\Lambda^2/\tau_1^6+ \sigma^2n^2\Lambda/\tau_1^4\right)}{T}\mper$$
	We can prove the bound for $G_C$ similarly. 
	\end{proof}

\begin{lemma}\label{claim:quadratic_form_variance}
	Let $x_1,\dots,x_T$ be independent random variables with mean 0 and variance 1 and 4-th moment bounded by $O(1)$, and $u_{ij}$ be vectors for $i,j\in [T]$. Moreover, let $\xi_1,\dots,\xi_T$ be independent random variables with mean 0 and variance $\sigma^2$ and $u'_{ij}$ be vectors for $i, j\in [T]$. Then, 
	$$\Var\left[\textstyle\sum_{i,j} x_ix_ju_{ij} + \sum_{i,j} x_i\xi_ju'_{ij}\right] \le O(1)\textstyle\sum_{i,j} \|u_{ij}\|^2 + O(\sigma^2)\sum_{i,j} \|u'_{ij}\|^2\,. $$
		\end{lemma}

\begin{proof}
	
	Note that the two sums in the target are independent with mean 0, 
		therefore we only need to bound the variance of both sums individually. The proof follows  the linearity of expectation and the independence of $x_i$'s: 
	\begin{align*}
	\Exp\left[\left\|\textstyle\sum_{i, j} x_ix_ju_{ij}\right\|^2\right] & = \sum_{i,j}\sum_{k,\ell} \Exp\left[x_ix_jx_kx_{\ell}u_{ij}^{\top} u_{k\ell}\right] \\
	&= \sum_{i} \Exp[u_{ii}^{\top}u_{ii} x^4_{i}] + \sum_{i\neq j}{\Exp[u_{ii}^{\top}u_{jj}x_i^2x_j^2]}+ \sum_{i, j} \Exp\left[x_i^2x_j^2 (u_{ij}^{\top}u_{ij} + u_{ij}^{\top}u_{ji})\right]\\
	& \le \sum_{i,j}u_{ii}^{\top}u_{jj}  + O(1)\sum_{i,j} \|u_{ij}+u_{ji}\|^2 \\
	& = \left\|\textstyle\sum_{i} u_{ii}\right\|^2 + O(1)\sum_{i,j} \|u_{ij}\|^2 
	\end{align*}
	where at second line we used the fact that for any monomial $x^{\alpha}$ with an odd degree on one of the $x_i$'s, $\Exp[x^{\alpha}] = 0$. 
	Note that $\Exp[\sum_{i, j} x_ix_ju_{ij}] = \sum_{i} u_{ii}$. Therefore, 
	\begin{equation}
	\Var\left[\textstyle\sum_{i, j} x_ix_ju_{ij}\right]= \Exp\left[\|\textstyle\sum_{i, j} x_ix_ju_{ij}\|^2\right]  - \|\Exp[\textstyle\sum_{i, j} x_ix_ju_{ij}] \|^2 \le O(1)\sum_{i,j} \|u_{ij}\|^2 \label{eqn:eqn5}
	\end{equation}	
	Similarly, we can control $\Var\left[\sum_{i,j} x_i\xi_ju'_{ij}\right]$ by $O(\sigma^2)\sum_{i,j}\|u_{ij}'\|^2$. 	\end{proof}

\section{Missing proofs in Section~\ref{sec:improper}}\label{sec:improper_appendix}

\subsection{Proof of Lemma~\ref{lem:approximation_inverse}}
Towards proving Lemma~\ref{lem:approximation_inverse}, we use the following lemma to express the inverse of a polynomial as a sum of inverses of degree-1 polynomials. 
\begin{lemma}\label{lem:inverse_expansion}
	Let $p(z) = (z-\lambda_1)\dots (z-\lambda_n)$ where $\lambda_j$'s are distinct. Then we have that 
	\begin{equation}
	\frac{1}{p(z)} = \sum_{j=1}^n \frac{t_j}{z-\lambda_j}\,,\qquad
	\text{where}\quad t_j = \left(\textstyle\prod_{i\neq j}(\lambda_j-\lambda_i)\right)^{-1} 
	\mper
	\label{eqn:inverse_expansion}
	\end{equation}
\end{lemma}

\begin{proof}[Proof of Lemma~\ref{lem:inverse_expansion}]
	By interpolating constant function at points $\lambda_1,\dots, \lambda_n$ using Lagrange interpolating formula, we have that 
	\begin{equation}
	1 = \sum_{j=1}^n \frac{\prod_{i\neq j}(x-\lambda_i)}{\prod_{i\neq j}(\lambda_j-\lambda_i)}\cdot 1
	\end{equation}
	Dividing $p(z)$ on both sides we obtain equation~\eqref{eqn:inverse_expansion}. 
\end{proof}

The following lemma computes the Fourier transform of function $1/(z-\lambda)$. 
\begin{lemma}\label{lem:zmoverzminuslambda}
	Let $m \in \mathbb{Z}$, and $\unitcircle$ be the unit circle in complex plane, and $\lambda\in \C$ inside the $\unitcircle$. Then we have that 
	\begin{equation}
	\int_{\unitcircle} \frac{z^m}{z-\lambda} \,\mathrm{d}z = \left\{\begin{array}{ll}
	2\pi i\lambda^m & \textup{ for } m \ge 0\\
	0 & \textup{ o.w.}\end{array}\right.\nonumber
	\end{equation}
\end{lemma}

\begin{proof}[Proof of Lemma~\ref{lem:zmoverzminuslambda}]
	For $m \ge 0$, since $z^m$ is a holomorphic function, by Cauchy's integral formula, we have that 
	\begin{equation}
	\int_{\unitcircle} \frac{z^m }{z-\lambda}\,\mathrm{d}z= 2\pi i\lambda^m\mper\nonumber
	\end{equation} 
	For $m < 0$, by changing of variable $y = z^{-1}$ we have that 
	\begin{equation}
	\int_{\unitcircle} \frac{z^m}{z-\lambda}\,\mathrm{d}z= \int_{\unitcircle} \frac{y^{-m-1}}{1-\lambda y}\,\mathrm{d}y\mper\nonumber
	\end{equation} 
	since $|\lambda y| = |\lambda| < 1$, then we by Taylor expansion we have, 
	\begin{align}
	\int_{\unitcircle} \frac{y^{-m-1}}{1-\lambda y}\,\mathrm{d}y & = \int_{\unitcircle} y^{-m-1}\left(\sum_{k=0}^{\infty} (\lambda y)^k\right) \,\mathrm{d}y\mper\nonumber
	\end{align}
	Since the series $\lambda y$ is dominated by $|\lambda|^k$ which converges, we can switch the integral with the sum. Note that $y^{-m-1}$ is holomorphic for $m < 0$,  and therefore we conclude that 
	\begin{align}
	\int_{\unitcircle} \frac{y^{-m-1}}{1-\lambda y}\,\mathrm{d}y & = 0 \mper\nonumber\qedhere
	\end{align}
\end{proof}

Now we are ready to prove Lemma~\ref{lem:approximation_inverse}. 

\begin{proof}[Proof of Lemma~\ref{lem:approximation_inverse}]
	Let $m = n+d$. We compute the Fourier transform of $z^m/p(z)$. That is, we write 
	\begin{equation}
	\frac{e^{im\theta}}{p(e^{i\theta})} = \sum_{k=-\infty}^{\infty} \beta_ke^{ik\theta}\mper \nonumber 
	\end{equation}
	where $$\beta_k = \frac{1}{2\pi}\int_{0}^{2\pi} \frac{e^{i(m-k)\theta}}{p(e^{i\theta})}d\theta
	 	 =  \frac{1}{2\pi i}\int_{\unitcircle}\frac{z^{m-k-1}}{p(z)}\,\mathrm{d}z$$
	By Lemma~\ref{lem:inverse_expansion}, we write 
	\begin{equation}
	\frac{1}{p(z)} = \sum_{j=1}^n \frac{t_j}{z-\lambda_j}\mper\nonumber	\end{equation}
	Then it follows that 
	\begin{align*}
	\beta_k & = \frac{1}{2\pi i}\sum_{j=1}^nt_j\int_{\unitcircle}\frac{z^{m-k-1}}{z-\lambda_j}\,\mathrm{d}z
	\end{align*}
	Using Lemma~\ref{lem:zmoverzminuslambda}, we obtain that 
	\begin{equation}
	\beta_k = 	\left\{\begin{array}{ll}
	\sum_{j=1}^{n} t_j \lambda_j^{m-k-1} & \textup{ if } -\infty \le  k\le m-1\\
	0 & \textup{ o.w.}
	\end{array}\right.
	\end{equation}
	We claim that 
	\begin{equation}
	\sum_{j=1}^{n}t_j\lambda_j^{n-1} = 1\,,\qquad\text{and}\qquad
	\sum_{j=1}^{n}t_j\lambda_j^s = 0\,,\quad 0\le s < n-1\mper\nonumber
	\end{equation}
	Indeed these can be obtained by writing out the lagrange interpolation for polynomial $f(x) = x^s$ with $s\le n-1$ and compare the leading coefficient. Therefore, we further simplify $\beta_k$ to 
	\begin{equation}
	\beta_k = 	\left\{\begin{array}{ll}
	\sum_{j=1}^{n} t_j \lambda_j^{m-k-1}  & \textup{ if } -\infty < k < m-n \\
	1 & \textup{ if } k = m-n\\
	0 & \textup{ o.w.}
	\end{array}\right.
	\end{equation}
	Let $h(z) = \sum_{k\ge 0} \beta_k z^k$.  Then we have that $h(z)$ is a polynomial with degree $d = m-n$ and leading term 1. Moreover, for our choice of~$d,$
	\begin{align*}
	\left|\frac{z^m}{p(z)} - h(z)\right| & = \left|\sum_{k < 0} \beta_k z^k\right| \le \sum_{k < 0}|\beta_k| \le \max_{j}|t_j|(1-\lambda_j)^n \sum_{k< 0}(1-\gamma)^{d-k-1} \\
	& \le \Gamma (1-\gamma)^d/\gamma < \zeta\,.\qedhere
	\end{align*}
\end{proof}

\subsection{Proof of Theorem~\ref{thm:improper_passive}}

Theorem~\ref{thm:improper_passive} follows directly from a combination of Lemma~\ref{lem:passive_reduction} and Lemma~\ref{lem:square_root_reduction} below. Lemma~\ref{lem:passive_reduction} shows that the denominator of a function (under the stated assumptions) can be extended to a polynomial that takes values in $\mathcal{C}^+$ on unit circle. Lemma~\ref{lem:square_root_reduction} shows that it can be further extended to another polynomial that takes values in $\mathcal{C}$.

\begin{lemma}\label{lem:passive_reduction}
	 Suppose the roots of $s$ are inside circle with radius $\alpha < 1$, and $\Gamma=\Gamma(s)$. If transfer function $G(z) = s(z)/p(z)$ satisfies that $G(z)\in \ballC_{\tau_0,\tau_1,\tau_2}$ (or $G(z)\in\ballC^+_{\tau_0,\tau_1,\tau_2}$) for any $z$ on unit circle, then there exists $u(z)$ of degree $d = O_{\tau}(\max\{(\frac{1}{1-\alpha}\log\frac{\sqrt{n}\Gamma\cdot \|p\|_{\mathcal{H}_2}}{1-\alpha}, 0\})$ such that $p(z)u(z)/z^{n+d}\in \ballC_{_{\tau_0',\tau_1',\tau_2'}}$ (or $p(z)u(z)/z^{n+d}\in \ballC^+_{\tau_0',\tau_1',\tau_2'}$ respectively) for $\tau' = \Theta_{\tau}(1)$ ,  where $O_\tau(\cdot), \Theta_{\tau}(\cdot)$ hide the polynomial dependencies on $\tau_0,\tau_1,\tau_2$. \end{lemma}

\begin{proof}[Proof of Lemma~\ref{lem:passive_reduction}]
	By the fact that $G(z) = s(z)/p(z)\in \ballC_{\tau_0,\tau_1,\tau_2}$, we have that $p(z)/s(z) \in \ballC_{\tau_0',\tau_1',\tau_2'}$ for some $\tau'$ that polynomially depend on $\tau$. 
	Using Lemma~\ref{lem:approximation_inverse}, there exists $u(z)$ of degree $d$ such that 
	\begin{equation}
	\left|\frac{z^{n+d}}{s(z)} - u(z)\right| \le \zeta \mper\nonumber
	\end{equation}
	where we set $\zeta \ll \min\{\tau_0',\tau_1'\}/\tau_2'\cdot \|p\|_{\mathcal{H}_{\infty}}^{-1}$. Then we have that 
	\begin{equation}
		\left|p(z)u(z)/z^{n+d} - \frac{p(z)}{s(z)}\right|\le |p(z)|\zeta\ll \min\{\tau_0',\tau_1'\}.\label{eqn:eqn23}
	\end{equation}
	It follows from equation~\eqref{eqn:eqn23} implies that that $p(z)u(z)/z^{n+d} \in \ballC_{\tau_0'',\tau_1'',\tau_2''}$, where $\tau''$ polynomially depends on $\tau$. The same proof still works when we replace $\ballC$ by $\ballC^+$. 
\end{proof}

\begin{lemma}\label{lem:square_root_reduction}
	Suppose $p(z)$ of degree $n$ and leading coefficient 1 satisfies that $p(z)\in \ballC^+_{\tau_0,\tau_1,\tau_2}$ for any $z$ on unit circle. Then there exists $u(z)$ of degree $d$ such that $p(z)u(z)/z^{n+d}\in \ballC_{\tau_0',\tau_1',\tau_2'}$ for any $z$ on unit circle with $d = O_{\tau}(n)$ and $\tau_0',\tau_1',\tau_2' = \Theta_{\tau}(1)$, where $O_\tau(\cdot), \Theta_{\tau}(\cdot)$ hide the dependencies on $\tau_0,\tau_1,\tau_2$. 
\end{lemma}

\begin{proof}[Proof of Lemma~\ref{lem:square_root_reduction}]
	We fix $z$ on unit circle first. Let's defined $\sqrt{p(z)/z^n}$ be the square root of $p(z)/z^n$ with principle value. Let's write $p(z)/z^n= \tau_2(1 + (\frac{p(z)}{\tau_2 z^n}-1))$ and we take Taylor expansion for $\frac{1}{\sqrt{p(z)/z^n}} = \tau_2^{-1/2} (1+(\frac{p(z)}{\tau_2 z^n} -1))^{-1/2} = \tau_2^{-1/2}\left(\sum_{k=0}^{\infty}(\frac{p(z)}{\tau_2 z^n}-1)^k\right)$. Note that since $\tau_1 < |p(z)| < \tau_2$, we have that $|\frac{p(z)}{\tau_2 z^n}-1| < 1-\tau_1/\tau_2$. Therefore truncating the Taylor series at $k = O_\tau(1)$ we obtain a polynomial a rational function $h(z)$ of the form
	$$\textstyle h(z) = \sum_{j\ge 0}^{k}(\frac{p(z)}{\tau_2 z^n}-1)^j\,,$$ 
	which approximates $\frac{1}{\sqrt{p(z)/z^n}}$ with precision $\zeta \ll \min\{\tau_0,\tau_1\}/\tau_2$, that is, 
$\left|\frac{1}{\sqrt{p(z)/z^n}}-h(z)\right|\le \zeta\mper\nonumber$
	Therefore, we obtain that 
		$\left|\frac{p(z)h(z)}{z^n}-\sqrt{p(z)/z^n}\right|\le \zeta|p(z)/z^n|\le \zeta \tau_2\mper\nonumber$
		Note that since $p(z)/z^n\in \ballC^{+}_{\tau_0,\tau_1,\tau_2}$, we have that $\sqrt{p(z)/z^n}\in \ballC_{\tau_0',\tau_1',\tau_2'}$ for some constants $\tau_0',\tau_1',\tau_2'$. Therefore $\frac{p(z)h(z)}{z^n} \in \ballC_{\tau_0',\tau_1',\tau_2'}$. Note that $h(z)$ is not a polynomial yet. Let $u(z) = z^{nk}h(z)$ and then $u(z)$ is polynomial of degree at most $nk$ and $p(z)u(z)/z^{(n+1)k}\in  \ballC_{\tau_0',\tau_1',\tau_2'}$ for any $z$ on unit circle. 
\end{proof}

\section{Back-propagation implementation}
In this section we give a detailed implementation of using back-propagation to compute the gradient of the loss function. The algorithm is for general MIMO case with the parameterization~\eqref{eqn:mimo_controllable_form}. To obtain the SISO sub-case, simply take $\lin = \lout = 1$.

\begin{algorithm}
	\begin{algorithmic}[]
		\STATE {\bf Parameters}: $\hata \in \R^n$, $\hatC \in \R^{\lin\times n\lout}$, and $\hatD\in \R^{\lin\times \lout}$. Let $\hatA = \MCC(\hata) = \CC(\hata)\otimes \Id_{\lin}$ and $\hatB = e_n\otimes \Id_{\lin}$. 
		\STATE {\bf Input: } samples $((x^{(1)},y^{1}),\dots, x^{(N)},y^{(N)})$ and projection set $\ball_{\alpha}$. 
		\vspace{.1in}
		\FOR{each sample $(x^{(j)},y^{j}) = ((x_1,\dots,x_T),(y_1,\dots,y_T))$} 
		\STATE {\bf Feed-forward pass}: 
		\STATE \hspace{\algorithmicindent} $h_0 = 0\in \R^{n\lin}$. 
		\STATE \hspace{\algorithmicindent} {\bf for }$k = 1$ to $T$ 
		\STATE \hspace{\algorithmicindent} \hspace{\algorithmicindent} $\hath_k \leftarrow \hatA h_{k-1} + \hatB x_k$, $\haty_t \leftarrow \hatC h_k + \hatD x_k$ and $\hath_k \leftarrow \hatA h_{k-1} + \hatB x_k$. 
		\STATE \hspace{\algorithmicindent} {\bf end for}
		\STATE {\bf Back-propagation: }
		\STATE \hspace{\algorithmicindent} $\dh_{T+1} \leftarrow 0$, $G_A \leftarrow 0$, $G_C \leftarrow 0$. $G_D \leftarrow 0$
		\STATE \hspace{\algorithmicindent} $T_1 \leftarrow T/4$ 		\STATE \hspace{\algorithmicindent} {\bf for } $k = T$ to $1$ 
		\STATE \hspace{\algorithmicindent} \hspace{\algorithmicindent} if $k > T_1$, $\dy_k \leftarrow \haty_k - y_k$, o.w. $\dy_k \leftarrow 0$. Let  $\dh_k \leftarrow \hatC^{\top}\dy_k + \hatA^{\top}\dh_{k+1}$. 
		\STATE \hspace{\algorithmicindent} \hspace{\algorithmicindent} update $
		G_C  \leftarrow G_C + \frac{1}{T-T_1}\dy_k \hath_k,~~
		G_A \leftarrow  G_A -  \frac{1}{T-T_1}B^{\top}\dh_k \hath_{k-1}^{\top}, $ and $
		G_D \leftarrow  G_D + \frac{1}{T-T_1}\dy_k x_k. 
		$
				\STATE \hspace{\algorithmicindent} {\bf end for}
				\STATE {\bf Gradient update: } $\hatA \leftarrow \hatA - \eta \cdot G_A$, $\hatC \leftarrow \hatC - \eta \cdot G_C$, $\hatD \leftarrow \hatD - \eta \cdot G_D$. 
		\STATE {\bf Projection step}: Obtain $\hata$ from $\hatA$ and set 
		$\hata \leftarrow \proj{\ball}(\hata)$, and $\hatA = \MCC(\hata)$
				\ENDFOR
	\end{algorithmic}
	\caption{Back-propagation}\label{alg:back_prop_general}
\end{algorithm}

\section{Projection to the set $\ball_{\alpha}$}\label{sec:projection}

In order to have a fast projection algorithm to the convex set $\ball_{\alpha}$, we consider a grid $\mathcal{G}_M$ of size $M$ over the circle with radius $\alpha$. We will show that $M = O_{\tau}(n)$ will be enough to approximate the set $\ball_{\alpha}$ in the sense that projecting to the approximating set suffices for the convergence. 

Let $\ball_{\alpha,\tau_0,\tau_1,\tau_2}' = \{a: p_a(z)/z^n \in \ballC_{\tau_0,\tau_1,\tau_2}, \forall z\in \mathcal{G}_M\}$ and $\ball_{\alpha,\tau_0,\tau_1,\tau_2} = \{a: p_a(z)/z^n \in \ballC_{\tau_0,\tau_1,\tau_2}, \forall |z| =\alpha\}$. Here $\ballC_{\tau_0,\tau_1,\tau_2}$ is defined the same as before though we used the subscript to emphasize the dependency on $\tau_i$'s, 
\begin{equation}
\ballC_{\tau_0,\tau_1,\tau_2} = \{z\colon \Re z \ge (1+\tau_0)|\Im z|\}
\cap \{z\colon \tau_1 < \Re z <
\tau_2\}\,.
\end{equation}

We will first show that with $M = O_{\tau}(n)$, we can make $\ball'_{\alpha, \tau_1,\tau_2,\tau_3}$ to be sandwiched within to two sets $\ball_{\alpha,\tau_0,\tau_1,\tau_2}$ and $\ball_{\alpha,\tau_0',\tau_1',\tau_2'}$. 
\begin{lemma}\label{lem:approximation_projection}
	For any $\tau_0 > \tau_0', \tau_1 > \tau_1',\tau_2 < \tau_2'$, we have that for $M = O_{\tau}(n)$, there exists $\kappa_0,\kappa_1,\kappa_2$ that polynomially depend on $\tau_i,\tau_i'$'s such that $\ball_{\alpha,\tau_0,\tau_1,\tau_2}\subset \ball_{\alpha,\kappa_0,\kappa_1,\kappa_2}'\subset \ball_{\alpha,\tau_0',\tau_1',\tau_2'}$ 
\end{lemma}

Before proving the lemma, we demonstrate how to use the lemma in our algorithm: We will pick $\tau_0' = \tau_0/2$, $\tau_1' = \tau_1/2$ and $\tau_2'  = 2\tau_2$, and find $\kappa_i$'s guaranteed in the lemma above. Then we use $\ball_{\alpha,\kappa_0,\kappa_1,\kappa_2}'$ as the projection set in the algorithm (instead of $\ball_{\alpha,\tau_0,\tau_1,\tau_2})$). First of all, the ground-truth solution $\Theta$ is in the set $\ball'_{\alpha,\kappa_0,\kappa_1,\kappa_2}$. Moreover, since $\ball_{\alpha,\kappa_0,\kappa_1,\kappa_2}'\subset \ball_{\alpha,\tau_0',\tau_1',\tau_2'}$, we will guarantee that the iterates $\hatTheta$ will remain in the set $\ball_{\alpha,\tau_0',\tau_1',\tau_2'}$ and therefore the quasi-convexity of the objective function still holds\footnote{with a slightly worse parameter up to constant factor since $\tau_i$'s are different from $\tau_i$'s up to constant factors}. 

Note that the set $\ball'_{\alpha, \kappa_0,\kappa_1,\kappa_2}$ contains $O(n)$ linear constraints and therefore we can use linear programming to solve the projection problem. Moreover, since the points on the grid forms a Fourier basis and therefore Fast Fourier transform can be potentially used to speed up the projection. 
Finally, we will prove Lemma~\ref{lem:approximation_projection}. We need S. Bernstein's inequality for polynomials. 

\begin{theorem}[Bernstein's inequality, see, for example, \cite{schaeffer1941}]
	Let $p(z)$ be any polynomial of degree $n$ with complex coefficients. Then,
	\begin{equation*}
	\sup_{|z|\leq 1} |p'(z)|  \leq n  \sup_{|z|\leq 1} |p(z)|.
	\end{equation*}
\end{theorem}

We will use the following corollary of Bernstein's inequality. 
\begin{corollary}\label{cor:berstein}
		Let $p(z)$ be any polynomial of degree $n$ with complex coefficients. Then, for $m = 20 n$, 
		\begin{equation*}
		\sup_{|z|\leq 1} |p'(z)|  \leq 2n  \sup_{k\in [m]} |p(e^{2ik\pi/m})| .
		\end{equation*}
\end{corollary}

\begin{proof}
	For simplicity let $\tau = \sup_{k\in [m]} |p(e^{2ik\pi/m})|$, and let $\tau' = \sup_{k\in [m]} |p(e^{2ik\pi/m})|$. If $\tau' \le 2\tau$ then we are done by Bernstein's inequality. Now let's assume that $\tau' > 2\tau$. Suppose $p(z) = \tau'$. Then there exists $k$ such that $|z-e^{2\pi i k /m}| \le 4/m$ and $|p(e^{2\pi i k /m})|\le \tau$. Therefore by Cauchy mean-value theorem we have that there exists $\xi$ that lies between $z$ and $e^{2\pi i k /m}$ such that $p'(\xi)\ge m(\tau'-\tau)/4\ge 1.1n\tau'$, which contradicts Bernstein's inequality.	\end{proof}
\begin{lemma}\label{lem:lem1}
	Suppose a polynomial of degree $n$ satisfies that $|p(w)|\le \tau$ for every $w = \alpha e^{2i\pi k/m}$ for some $m \ge 20n$. Then for every $z$ with $|z|=\alpha$ there exists $w = \alpha e^{2i\pi k/m}$ such that $|p(z) - p(w)|\le O(n\alpha\tau/m)$. 
\end{lemma}

\begin{proof}
	Let $g(z) = p(\alpha z)$ by a polynomial of degree at most $n$. Therefore we have $g'(z) = \alpha p(z)$. Let $w = \alpha e^{2i\pi k/m}$ such that $|z-w|\le O(\alpha/m)$. Then we have
	\begin{align}
	|p(z)-p(w)| = |g(z/\alpha)-p(w/\alpha)| & \le \sup_{|x|\le 1}|g'(x)| \cdot \frac{1}{\alpha}|z-w| \tag{By Cauchy's mean-value Theorem} \\
	& \le \sup_{|x|\le 1}|p'(x)| \cdot |z-w| \le n\tau |z-w|\mper \tag{Corallary~\ref{cor:berstein}}\\
	& \le O(\alpha n\tau/m)\mper\nonumber
	\end{align}
	\end{proof}
Now we are ready to prove Lemma~\ref{lem:approximation_projection}. 
\begin{proof}[Proof of Lemma~\ref{lem:approximation_projection}]
	We choose $\kappa_i = \frac{1}{2}(\tau_i+\tau_i')$.The first inequality is trivial. We prove the second one.  	 Consider $a$ such that  $a \in \ball_{\alpha,\kappa_0,\kappa_1,\kappa_2}$. We wil show that $a \in \ball'_{\alpha, \tau_0',\tau_1',\tau_2'}$. Let $q_a(z) = p(z^{-1})z^n$. By Lemma~\ref{lem:lem1}, for every $z$ with $|z| = 1/\alpha$, we have that there exists $w = \alpha^{-1}e^{2\pi i k /M }$ for some integer $k$ such that $|q_a(z)-q_a(w)|\le O(\tau_2n/(\alpha M))$. Therefore let $M = cn$ for sufficiently large $c$ (which depends on $\tau_i$'s), we have that for every $z$ with $|z|=1/\alpha$,  $q_a(z)\in \ballC_{\tau_0',\tau_1',\tau_2'}$. This completes the proof. 
	\end{proof}

\end{document}